\documentclass{article}

\usepackage{PRIMEarxiv}

\usepackage[utf8]{inputenc} % allow utf-8 input
\usepackage[T1]{fontenc}    % use 8-bit T1 fonts
\usepackage{hyperref}       % hyperlinks
\usepackage{url}            % simple URL typesetting
\usepackage{booktabs}       % professional-quality tables
\usepackage{amsfonts}       % blackboard math symbols
\usepackage{nicefrac}       % compact symbols for 1/2, etc.
\usepackage{microtype}      % microtypography
\usepackage{lipsum}
\usepackage{fancyhdr}       % header
\usepackage{graphicx}       % graphics
\graphicspath{{media/}}     % organize your images and other figures under media/ folder
\usepackage{nicefrac}       % compact symbols for 1/2, etc.
\usepackage{microtype}      % microtypography
\usepackage{xcolor}         % colors
\usepackage{mathtools}
\usepackage{subfig}
\usepackage{graphicx}

\usepackage{algorithm}
\usepackage{algpseudocode}
\usepackage{svg}
\usepackage{paralist} % Add this line

\usepackage{comment}
\usepackage{ulem}
\usepackage{tikz}

\usetikzlibrary{arrows.meta,
                calc, chains,
                ext.paths.ortho,      
                positioning,
                quotes}

\newif\ifcomments \commentsfalse

\usepackage{soul}

\newcommand{\sn}[1]{
\ifcomments
{\textcolor{blue}{[{ {\bf SVN:} { {#1}}}]}}
\fi
}

\newcommand{\luk}[1]{
\ifcomments
{\textcolor{purple}{[{ {\bf LO:} { {#1}}}]}}
\fi
}

\usepackage{array,multirow}
\usepackage{makecell}
\usepackage[flushleft]{threeparttable}
\usepackage{wrapfig}
\usepackage{colortbl}

\usepackage{duckuments}
\usepackage{amsthm}
\usepackage{amsmath}
\newtheorem{theorem}{Theorem}[section]
\newtheorem{lemma}[theorem]{Lemma}

\usepackage{hyperref}
\usepackage{url}

\newcommand{\dist}{D_\theta}
\newcommand{\xtrain}{X_{T}}
\newcommand{\gandist}{\mathcal{G}}
\newcommand{\mixgandist}{\mathcal{M}}
\newcommand{\midist}{\mathcal{R}}
\newcommand{\fpr}{\text{FPR}}
\newcommand{\tpr}{\text{TPR}}

\newcommand{\cifar}{\text{CIFAR-10}}
\newcommand{\detector}{\text{Detector}}

\newcommand{\gan}{G}

\newcommand{\refdistro}{\mathcal{P}}
\newcommand{\discriminator}{D'}

\newcommand{\datadistro}{\mathcal{P}}

\DeclareMathOperator*{\argmin}{argmin}

\title{Black-Box Training Data Identification in GANs via Detector Networks}
\author{
  Lukman Olagoke \\
  Harvard University \\
  \texttt{lolabisiolagoke@hbs.edu} \\
   \And
  Salil Vadhan \\
  Harvard University \\
  \texttt{svadhan@g.harvard.edu} \\
   \And
  Seth Neel \\
  Harvard University \\
  \texttt{sneel@hbs.edu} \\
}

\begin{document}
\maketitle

\begin{abstract}
Since their inception Generative Adversarial Networks (GANs) have been popular generative models across images, audio, video, and tabular data. In this paper we study whether given access to a trained GAN, as well as fresh samples from the underlying distribution, if it is possible for an attacker to efficiently identify if a given point is a member of the GAN's training data. This is of interest for both reasons related to copyright, where a user may want to determine if their copyrighted data has been used to train a GAN, and in the study of data privacy, where the ability to detect training set membership is known as a membership inference attack. Unlike the majority of prior work this paper investigates the privacy implications of using GANs in black-box settings, where the attack only has access to samples from the generator, rather than access to the discriminator as well. We introduce a suite of membership inference attacks against GANs in the black-box setting and evaluate our attacks on image GANs trained on the CIFAR10 dataset and tabular GANs trained on genomic data. Our most successful attack, called The $\detector$, involve training a second network to score samples based on their likelihood of being generated by the GAN, as opposed to a fresh sample from the distribution. We prove under a simple model of the generator that the detector is an approximately optimal membership inference attack. Across a wide range of tabular and image datasets, attacks, and GAN architectures, we find that adversaries can orchestrate non-trivial privacy attacks when provided with access to samples from the generator. At the same time, the attack success achievable against GANs still appears to be lower compared to other generative and discriminative models; this leaves the intriguing open question of whether GANs are in fact more private, or if it is a matter of developing stronger attacks. 
\end{abstract}
%\tableofcontents
\clearpage

\section{Introduction}
Beginning in 2014 with the seminal paper ~\cite{goodfellow14}, until the recent emergence of diffusion models ~\cite{diffusion1, diffusion2, diff_audio, diff_video}, generative adversarial networks (GANs) have been the dominant generative model across image, audio, video, and tabular data generation ~\cite{sotaImage, sotaAudio,vid-image,sotavidvid,tabdataModelling,effectivetabdata}, and are still in some cases competitive with diffusion models, for example in~\cite{gigagan,Karras2021,vqgan,gansformer}. One promising use case of generative models is when the underlying training data is private, and so there are restrictions on sharing the data with interested parties. Rather than sharing the data directly, after training a generative model $G$ on the private data, either $G$ can be shared directly, or a new synthetic dataset can be generated from $G$, which can be used for downstream tasks without sharing the underlying data. While this seems private on the surface, a line of work has shown that sharing access to a generative model or synthetic dataset may leak private information from the underlying training data ~\cite{chen_gan-leaks_2020,LOGAN2}. Recent work  ~\cite{extract_diff} has shown that as compared to GANs, diffusion models are far less \emph{private}, memorizing as much as  $5 \times$ more of their training data, and as a result are highly susceptible to attacks that can reconstruct the underlying training data. These considerations make GANs a more natural choice of generative model for highly sensitive  data -- or do they?

Most existing work probing the privacy properties of GANs has typically focused on the "white-box" setting, where an adversary has access to the discriminator of the GAN, and is able to use standard techniques to conduct membership inference attacks based on the loss ~\cite{chen_gan-leaks_2020}; a notable exception is~\cite{recent}. While these white-box attacks demonstrate releasing the discriminator in particular leaks information about the underlying training set, they are unsatisfying from a practical perspective, since typically only the generator would be shared after training. This raises a natural question about the privacy of GANs: \\
\\
\begin{centering}
    \emph{In the black-box setting where an adversary only has sample access to the generator, are effective privacy attacks possible}?
\end{centering} \\
\\

Another potential motivation for developing efficient methods to determine whether a point was used to train a GAN comes from copyright. Suppose the creator of an image for example, suspects that their image was used in the training set of a generative model without their permission. Since the vast majority of state-of-the-art generative models are only made available via access to the generator, applying a black-box membership inference attack is the only way the creator can argue it is (statistically) likely their image was part of the training set. We address this gap by developing a battery of membership inference attacks against GANs that only require access to generated samples. We evaluate our attacks on target GANs trained on image and tabular data, focusing on image GANs trained on the popular CIFAR10 dataset ~\cite{cifar}, and on tabular GANs trained on SNP (Single Nucleotide Polymorphisms ~\cite{snp}) data from two popular genomic databases ~\cite{dbgap, 10002015global}. We selected genomic data because of the obvious underlying privacy concerns, and because of the high-dimensionality of the data, which makes the setting more challenging from a privacy perspective. Early privacy attacks on genomic databases \cite{homer} led the NIH to take the extreme step of restricting access to raw genomic summary results computed on SNP datasets like dbGaP \cite{dbgap}, although they later reversed course following scientific outcry.

\paragraph{Contributions.} 
In this work, we develop state-of-the-art privacy attacks against GANs under the realistic assumption of black-box access, and conduct a thorough empirical investigation across $2$ GAN architectures trained on $6$ different tabular datasets created by sub-sampling $2$ genomic data repositories, and $4$ GAN architectures trained to generate images from $\cifar$. We develop a new attack, which we call the \emph{Detector} (and its extended version called the \emph{Augmented Detector or ADIS} ), that trains a second network to detect GAN-generated samples, and applies it to the task of membership inference. We also derive theoretical results (Theorem~\ref{thm:opt}) that shed light on our attack performance. We compare our $\detector$-based methods to an array of distance and likelihood-based attacks proposed in prior work \cite{chen_gan-leaks_2020, LOGAN, refdata} which we are the first to evaluate with recent best practice metrics for MIAs \cite{firstprinciple}. Generally, we find that while privacy leakage from GANs seems to be lower than reported in prior work \cite{LOGAN, extract_diff} for other types of generative models, there does still appear to be significant privacy leakage even in the black-box setting. Moreover, this privacy-leakage as measured by the success of our MIAs seems to vary based on the type of GAN trained, as well as the dimension and type of the underlying training data. To our knowledge, we are the first to apply detectors for generated samples to the task of membership inference against generative models.

\section{Related Work}
\emph{Membership Inference Attacks} or MIAs were defined by \cite{homer} for genomic data, and formulated by \cite{shokri_membership_2017} against ML models. In membership inference, an adversary uses model access as well as some outside information to determine whether a candidate point $x$ is a member of the model training set. Standard MIAs typically exploit the intuition that a point $x$ is more likely to be a training point if the loss of the model evaluated at $x$ is small \cite{yeom_privacy_2018}, or in the white-box setting where they have access to the model's gradients, if the gradients are small \cite{shokgrad}, although other approaches that rely on a notion of self-influence \cite{self-mi} and distance to the model boundary have been proposed \cite{neel, pap}. The state of the art attack against discriminative classifiers, called LiRA \cite{firstprinciple}, approximates the likelihood ratio of the loss when $x$ is in the training set versus not in the training set, training additional shadow models to sample from these distributions. 

Assuming access to the discriminator's loss, as in \cite{LOGAN}, loss-based attacks are easily ported over to GANs. However, given that the discriminator is not required to sample from the generator once training is finished, there is no practical reason to assume that the adversary has access to the discriminator. \cite{LOGAN} find that by thresholding on the discriminator loss, TPR at FPR = $.2$ on DCGAN trained on $\cifar$ is a little under $.4$ when the target GAN is trained for $100$ epochs, but goes up to almost $100\%$ after $200$ epochs, indicating severe over-fitting by the discriminator. When they employ a black-box attack, they are able to achieve TPR = $.37$ at FPR = $.2$.
We note that this is a less than $2$x improvement over baseline, even in the presence of potentially severe over-fitting during training.

\cite{Mi-montecarlo} propose a black-box attack on GANs trained on $\cifar$ that is very similar to a distance-based attack. Their statistic samples from the GAN and counts the proportion of generated samples that fall within a given distance $\epsilon$ of the candidate point. In order to set $\epsilon$, they estimate the $1\%$ or $.1\%$ quantiles of distances to the generated GAN via sampling. Note that this assumes access to a set of candidate points $x_i$ rather than a single candidate point that may or may not be a training point, although a similar idea could be implemented using reference data. While their attacks were quite effective against VAEs and on the simpler MNIST dataset, the best accuracy one of their distance-based attacks achieves on $\cifar$ is a TPR that is barely above $50\%$, at a massive FPR that is also close to $50\%$ (Figure $4$(c)) with their other distance-based method performing worse than random guessing. 
Relatedly, ~\cite{chen_gan-leaks_2020} proposed a distance-based attack scheme based on minimum distance to a test sample. In ~\cite{chen_gan-leaks_2020}\ the adversary has access to synthetic samples from the target GAN and synthetic samples from a GAN trained on reference data ($G_{ref}$).
We follow this approach for distance-based attacks, which we discuss further in Section~\ref{onedist}. The most closely related work to our detector methods is  ~\cite{refdata} who propose  a density-based model called DOMIAS (Detecting Overfitting for Membership
Inference Attacks against Synthetic Data), which infers membership by targeting local overfitting of the generative models. Rather than train a detector network to classify whether samples are generated from the target GAN or the reference data, DOMIAS performs dimension reduction in order to directly estimate both densities, and then uses the ratio of the densities as a statistic for membership inference. 

Related to our study of tabular GANs trained on genomic data, ~\cite{yelmen_creating_2021} proposed the use of GANs for the synthesis of realistic artificial genomes with the promise of none to little privacy loss. The absence of privacy loss in the proposed model was investigated by measuring the extent of overfitting using the nearest neighbor adversarial accuracy (AAT$_{TS}$) and privacy score metrics discussed in ~\cite{yale_privacy_2019}. It should be noted that while overfitting is sufficient to allow an adversary to perform MI (and hence constitute a privacy loss), it is not a necessary prerequisite for the attack to succeed, as shown in the formal analysis presented in ~\cite{yeom_privacy_2018}.

There have been a handful of recent papers that propose black-box attacks against generative models, nearly all of which rely on the heuristic that if a point is ``closer'' to generated samples that to random points from the distribution, it is more likely to be a training point \cite{Mi-montecarlo}\cite{chen_gan-leaks_2020}. The most closely related work to our detector methods is  ~\cite{refdata} who propose  a density-based model called DOMIAS, which infers membership by targeting local over-fitting of the generative models. Rather than train a detector network to classify whether samples are generated from the target GAN or the reference data, DOMIAS performs dimension reduction in order to directly estimate both densities, and then uses the ratio of the densities as a statistic for membership inference. 
We implement many of these methods in our experiments section, and summarize them in more detail in Appendix~\ref{domias_explained}.

\section{Preliminaries}
% The goal of Membership inference (MI) is to predict whether  a  data  point was used in training a model or not.\cite{shokri_membership_2017}. \textcolor{teal}{ more explanations to be added here.}
\textbf{Generative Adversarial Network(GANs).} 
Given training data drawn from a distribution, $\xtrain \sim \datadistro$, a generative model tries to approximate $\datadistro$. 
Generative adversarial networks (GANs) are examples of generative models and have gained popularity for their ability to generate realistic samples across a range of domains ~\cite{arjovsky_wasserstein_2017,yu_inclusive_2020,biggan,contragan,dcgan}.  
The basic GAN set-up consists of two players - the \emph{discriminator} ($\discriminator$) and the \emph{generator} ($\gandist$)- engaging  in a minimax game ~\cite{goodfellow_generative_2020} with training objective given as: 
$$
\mathbb{E}_{ \mathbf{x} \sim \datadistro} [\log \discriminator(\mathbf{x})] + \mathbb{E}_{\mathbf{z} \sim p_{\mathbf{z}}(\mathbf{z})} [\log (1 - \discriminator(\gan(\mathbf{z}))) $$
The generator, $\gan$, takes latent noise, $z$, as an input and generates a sample as output, while the discriminator examines data samples (i.e, $\mathbf{x},\,  \,\text{and} \, \, \gan(\mathbf{z})$) and tries to discriminate samples from the generator (i.e  $\gan(\mathbf{z})$ ) from the real data samples (i.e $\mathbf{x}$). The generator's target is to generate realistic samples that would fool the discriminator, while the discriminator tries to detect counterfeit samples from the generator~\cite{goodfellow_generative_2020}. The training objective is typically optimized by updating the generator and discriminator via simultaneous gradient ascent.

Throughout the paper we'll let $\mathcal{T}$ be the distribution that samples a random point from $\xtrain$, $\gandist$ the distribution that samples a point from $\gan$, $\mixgandist = \frac{1}{2}\gandist + \frac{1}{2}\datadistro$ be the mixture distribution of $\gandist$ and $\datadistro$, which we will need to define our detector attack in Section~\ref{sec:distinguisher}, and $\midist = \frac{1}{2}\mathcal{T} + \frac{1}{2}\datadistro$ be the mixture distribution we evaluate our MIAs on. Given a point $x$ let $\gandist(x), \datadistro(x)$ denote the respective probability densities evaluated at $x$. 

\paragraph{Black-Box Setting.}
% Observe that from the given attack framework in previous paragraph, the target GAN training data, $\xtrain$, and the reference data $\refdistro$  are from the same underlying distribution $\pi$. 
% Given the attack framework above, observe that $\xtra$ is unknown to the adversary, just as the sample composition of the test set $\query$ is also unknown to the adversary. Similarly, the adversary has complete access to $X_{R}$. However,  as the  attack tuple $(\mathcal{T}, X_{R},G(\cdot))$ indicates,  the adversary has access to synthetic samples  from the GAN rather than the underlying model weights. We refer to this as the black-box GAN setting and denote it by $G(\cdot)$.
Following convention, we assume that our black-box attacks have access to samples from $\gandist$ but not the weights of $\gan$ directly, access to details about how the GAN was trained (e.g. the architecture of course not the training samples), and fresh samples from the distribution $\refdistro$ \cite{shokri_membership_2017, blackbox1, sablay}.
% As discussed in the previous section, there are different settings under which attacks against GANs could be set up - for example, access to the discriminator, $D$, could be assumed. An extensive explanation of these various settings is detailed in  \cite{chen_gan-leaks_2020}. We will focus on the black box setting where  no  access  to the generator  $G$, the discriminator $D$, and the latent variables
% $z$ are assumed. In addition, no access to the model architecture or underlying model weights is assumed. 
% However, we assume we have  unlimited access to the synthetic samples from the GAN's generator. Black-box access  to synthetic samples from the GAN would be denoted by the notation $G(\cdot)$.  In addition, similar to \cite{refdata}, we will assume access to a reference sample. The  reference sample is  similar in distribution but disjoint  from the training set of the GAN. This setting is close to realistic practical scenario \cite{refdata}. 
%\sn{Formalize that training data and reference samples are assumed to come from an underlying distribution $\pi$. Black-box means that an adversary only has access to samples from the model rather than the underlying weights. }
%\sn{Define the attack prior to defining black-box vs. whitebox attacks.}
\paragraph{Attack Framework.}\label{sec:frame}
We adopt an attack framework similar to that used in countless papers on membership inference attacks against machine learning models \cite{shokri_membership_2017, ye2021enhanced, firstprinciple}. The model developer 
samples a training set $\xtrain \sim \datadistro$, and trains a GAN $\gan$ on $\xtrain$, $\gan \sim \text{Train}(\xtrain)$. Now consider a membership inference attack (MIA) $\mathcal{A}$; draw $x' \sim \midist$ and send $x'$ to $\mathcal{A}$. The membership inference attacker  $\mathcal{A}$ receives the \emph{candidate point} $x'$ and then outputs a guess: $1$ if $x' \in \xtrain$, $0$ if $x' \not\in \xtrain$. Note that if $\mathcal{A}$ only receives $x'$, then it is impossible for $\mathcal{A}$ to output a guess that has accuracy $> \frac{1}{2}$, since the marginal distribution on $x'$ is always $\datadistro$. Absent any additional information, for an fixed FPR $\alpha$, the TPR achieved by $\mathcal{A}$'s guess is always $\leq \alpha$. In the black-box setting $\mathcal{A}$ also receives black-box access to draw random samples from $\gandist$ and additional samples from $\datadistro$. From the above discussion, any accuracy $\mathcal{A}$ is able to achieve that is better than a random guess must be due to the ability to infer if $x' \in \xtrain$ through access to samples from $\gandist$. In practice, rather than an MIA $\mathcal{A}$ outputting a $1$ or $0$ guess, $\mathcal{A}$ will output a real-valued score $\mathcal{A}(x')$ corresponding to how likely $x'$ is to be in $\xtrain$. Then any threshold $\tau \in \mathbb{R}$ corresponds to an attack $\textbf{1}\{\mathcal{A}(x') > \tau \}$, and varying $\tau$ traces out a receiver operating characteristic (ROC) curve of achievable FPR/TPRs \cite{firstprinciple}.
    
\paragraph{Attack Evaluation Metrics.} 
Recent work \cite{firstprinciple, liu2022} argues the most meaningful way to evaluate a membership inference attack $\mathcal{A}$ is not to look at the overall accuracy, but to focus on what TPRs are achieveable at low FPRs, as this could correspond to a more realistic privacy violation: a small subset of the training data that model can identify with high confidence. Note that this correspond to large values of $\frac{\text{TPR}}{\text{FPR}}$ at small FPRs. We adopt this convention here, reporting the achievable TPR at fixed FPRs ($.001, .005, .01, .1$) for all of our attacks (Tables~\ref{tab:summary}, \ref{tab:summary2}). As in \cite{firstprinciple} we plot all our attack ROC curves on the log-log scale in order to visualize what is happening at low FPRs. We also report the AUCs for each ROC curve in order to provide a quick high-level summary of attack success, but again we regard such average summary statistics as less important than the $\frac{\text{TPR}}{\text{FPR}}$ statistics.

\section{Attack Model Types}
The attacks we study fall into two broad categories: (i) $\detector$-based attacks rely on training a second classifier to distinguish generator samples from reference data samples, and then using that classifier's predicted probability of being generated by $\gandist$ as a proxy for training set membership (ii) Distance-based attacks compute the distance between the candidate point and samples from the generator and samples from the reference data, and predict the point is a training point if it is closer to generated samples. In Section~\ref{sec:distinguisher} we outline the basic process of training the $\detector$, and define its augmented variant ADIS. We state theoretical results that characterize the performance of an optimal detector-based attack and show that under a simplified model where the generator learns the distribution subject to some mode collapse, detector-based attacks are approximately \emph{optimal} MIAs. In Section~\ref{onedist} we define the distance-based attacks proposed in prior work that we evaluate our methods against. In our experiments we also compare against the DOMIAS attack proposed in \cite{refdata} -- we defer the details of our implementation to Appendix ~\ref{domias_explained}.

% Our first approach to MI attack, given query access to $G_z$ and $N$ test samples,  would be based on nearest neighbour sample inference as proposed in \cite{chen_gan-leaks_2020}.
\subsection{Detector-Based Attacks}\label{sec:distinguisher}

\begin{figure}
\label{fig:dist_train}
\centering
\begin{tikzpicture}[
node distance = 10mm, 
  start chain = A going right,
   arr/.style = {-Latex, semithick},
 block/.style = {minimum height=12mm, font=\small, 
                 text width=21mm, align=center, inner sep=1pt},
every edge quotes/.style = {auto, text width=19mm, align=center, 
                            font=\footnotesize}
                        ]
\setkeys{Gin}{width=16mm}
% nodes
    \begin{scope}[nodes={block, on chain=A, join = by arr}]
\node   {\includegraphics[width=1.8cm]{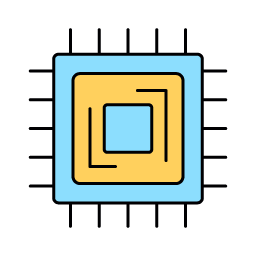}\\
         Generator\\ $G(\cdot)$};
\node   {\includegraphics[width=1.8cm]{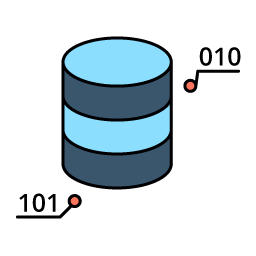}\\
        };
\node   {\includegraphics[width=1.8cm]{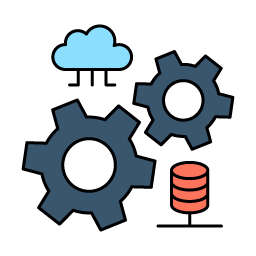}\\
         $(X_{\mathcal{D}}, 0) \cup (X_G, 1)$};
\node   {\includegraphics[width=1.8cm]{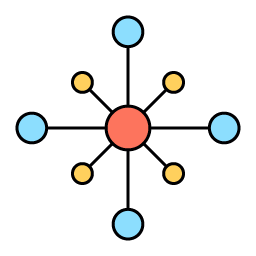}\\
         Train\\ Set};
\node   {\includegraphics[width=1.5cm]{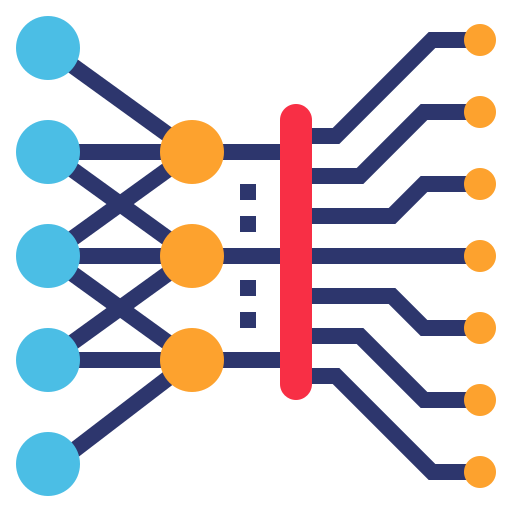}\\
         Detector \\ Training};
    \end{scope}
\node[above=of A-2, anchor=south] (ref) {\includegraphics[width=1.8cm]{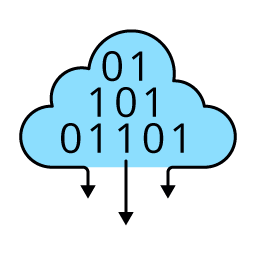}};
\node[above=of A-2, align=center, yshift=-3mm,xshift=11mm] (text) {\;\;\; \; \; Sample $X_{\mathcal{D}} \sim \mathcal{D}$};

\node [ above=of A-4]  (val) {\includegraphics[width=1.8cm]{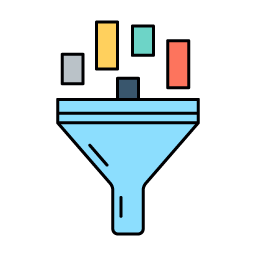}};    
\node[above=of A-4, align=center, yshift=-6mm,xshift=0mm] (text) {Validation \\ Set};
\node [right=of val]  (checker) {\includegraphics[width=1.9cm]{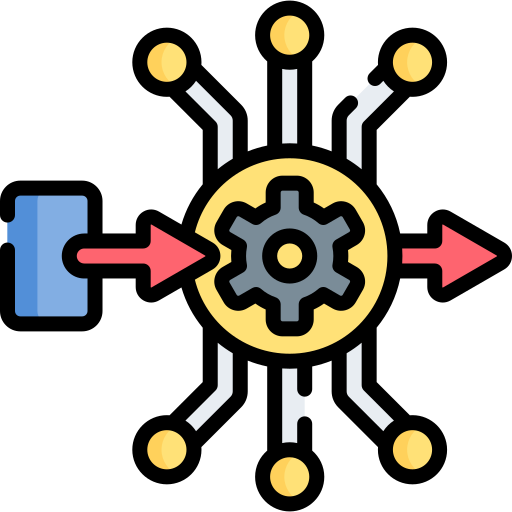}}  ;  {Check cycle};
\node[above=of A-5, align=center, yshift=-3mm,xshift=-12mm] (text) {Early \\ Stopping};
% Paths with text
\path   (A-1) to ["Generate synthetic data $X_G$"]    (A-2)
        (A-2) to ["Label and\\ Reshuffle Data"]                 (A-3)
        (A-3) to ["Split Training Data" ']          (A-4);
\draw[arr]  (ref) edge  (A-2)
            (val) edge  (checker)
            (A-5) edge  (checker)
             (checker) edge  (A-5)
            ($(A-3)!0.5!(A-4)$) |- (val);
    \end{tikzpicture} 
    \caption{The pipeline for training the $\detector$ $\dist$.}
    \label{tikz:disting} 
    \end{figure}
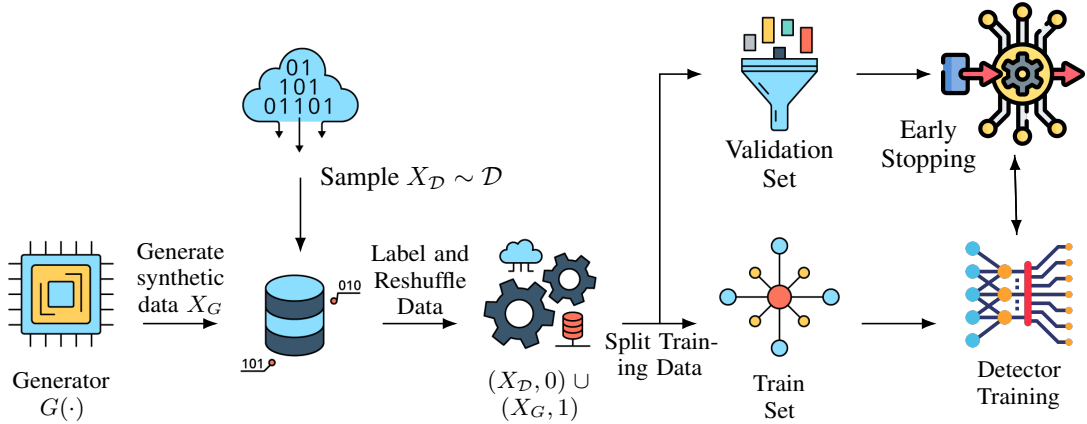
    
The $\detector$ is based on the premise that \emph{a network that can distinguish samples \emph{generated} from the GAN from samples from the true distribution, can also distinguish training samples from the distribution}. Specifically a network called the $\detector$ ($\dist$) is trained to classify samples from the GAN as $1$, and samples from $\datadistro$ as $0$. After training $\dist$, given a candidate point $x$, the membership inference score for $x$ is the predicted probability under $\dist$ of the point being generated from the GAN, which we'll denote $\dist(x)$. The high-level pipeline for training $\dist$ is shown in Figure~\ref{tikz:disting}, with the algorithmic details of how to set the hyper-parameters deferred to the relevant experimental sections. The variant of the $\detector$ attack we call the Augmented $\detector$ (ADIS) follows the same process of training the $\detector$, but augments the feature space used to train the $\detector$ with statistics based on the distance from $x$ to samples from the generator and reference data, including the DOMIAS statistic \cite{refdata}. The full details of how we train ADIS are deferred to Appendix~\ref{sec:adis}.

Now let $f: \mathcal{X} \to [0, 1]$, and $\text{Lip}(f)$ denote the Lipschitz constant of $f$. We can relate the error of $f$ when used for membership inference on points sampled from $\midist$ to the error of $f$ when used to classify whether points from $\mixgandist$ were sampled from $\gandist$ or $\datadistro$. If points sampled from $\datadistro$ are labeled $0$ and points from $\gandist$ are labeled $1$, then the FPR of $f$ used to classify points from $\mixgandist$ is $\fpr_{\mixgandist} =  \frac{1}{2}\mathbb{E}_{\datadistro}[f(x)]$. Now, the FPR of $f$ as an MIA is $\fpr_{\midist} = \mathbb{E}_{\midist}[f(x) | x \in \datadistro] = \fpr_{\mixgandist}$, so the FPRs are the same. The TPR is more interesting; it is easy to show (Lemma~\ref{lem:was}) that 
\begin{equation}
\label{eq:tpr_was}
\text{TPR}_{\midist}(f) \leq \text{TPR}_{\mixgandist}(f) + \frac{\text{Lip}(f)}{2}W^{1}(\gandist, \mathcal{T}), 
\end{equation}

where $W^1(\gandist, \mathcal{T})$ is the $1$-Wasserstein distance between $\gandist$ and $\mathcal{T}$. Note that when $\gandist = \mathcal{T}$ then $\midist = \mixgandist$, and the above inequality becomes equality, with $W^1(\gandist, \mathcal{T}) = 0$. Equation~\ref{eq:tpr_was} motivates the $\detector$ attack, because although we can't directly sample from $\midist$, and therefore can't train $f$ to maximize $\tpr_{\midist}(f)$ at a fixed $\fpr$ rate, in the black-box setting we can still sample from $\mixgandist$, and so we can train $f$ that maximizes $\tpr_{\mixgandist}(f)$ subject to $\fpr_{\mixgandist}(f) = \fpr_{\midist}(f) = \alpha$ for any $\alpha \in  [0, 1]$. While Equation~\ref{eq:tpr_was} is nice in that we can upper bound the rate we'd expect $\text{TPR}_{\midist}(f) \to \text{TPR}_{\mixgandist}(f)$ as $\gandist \to \mathcal{T}$, namely the rate $W^1(\gandist, \mathcal{T}) \to 0$, it doesn't tell us much about why in practice a classifier $f$ that minimizes error on $\mixgandist$ might achieve low error on $\midist$: (i) $f$ is a neural network and so $\text{Lip}(f)$ could be very large in practice, making Equation~\ref{eq:tpr_was} vacuous and (ii) $W^1(\gandist, \mathcal{T})$ can be quite large. 

In Theorem~\ref{thm:opt} we prove the neat result that if $\mathcal{G}$ is a mixture of the training set distribution $\mathcal{T}$ and the dataset $\datadistro$, and $f^*$ is the Bayes optimal classifier with respect to $\mathcal{M}$, then the MIA that thresholds based on $f^*$ achieves maximal $\tpr$ at any fixed $\fpr$ on the distribution $\midist$. We defer the proof to the Appendix, which uses the form of $\gandist$ to decompose the error on $\gandist$ into errors on $\mathcal{T}, \datadistro$, and follows from the Neyman Pearson lemma \cite{neyman} and the assumption that $f^*$ is Bayes optimal. 

\begin{theorem}
\label{thm:opt}
Suppose that $\gandist = \beta\datadistro + (1-\beta)\mathcal{T}$ for some $\beta \in [0,1]$. Let $f^*: \mathcal{X} \to [0, 1]$ be the posterior probability a sample from $\mixgandist$ is drawn from $\gandist: f^*(x) = \frac{\gandist(x)}{\gandist(x) + \mathcal{P}(x)}$. Then for any fixed $\alpha \in [0, 1], \; \exists \tau_{\alpha}$ such that attack 
$$
F^*(x) = \textbf{1}\{f^*(x) > \tau_{\alpha}\} 
$$
satisfies $\fpr_{\midist}(F^*) = \alpha$, and \emph{for any MIA $f$} such that $\fpr_{\midist}(f) = \alpha$,  $\tpr_{\midist}(f) \leq \tpr_{\midist}(F^*)$.
\end{theorem}
\begin{proof}
Let $f: \mathcal{X} \to \{0, 1\}$ be an arbitrary MIA achieving FPR $\alpha$, that is $\mathbb{E}_{x \sim \mathcal{D}}[f(x)] = \alpha$. Then the TPR of $f$ at distinguishing samples from $G$ is $\mathbb{E}_{x \sim G}[f(x)] = \beta \mathbb{E}_{x \sim \mathcal{\mathcal{T}}}[f(x)] + (1-\beta)\mathbb{E}_{x \sim \mathcal{D}}[f(x)]$, since $G = \beta\mathcal{T} + (1-\beta)\mathcal{D}$. Substituting in the FPR of $f$, we have $\mathbb{E}_{x \sim G}[f(x)] = \beta \mathbb{E}_{x \sim \mathcal{\mathcal{T}}}[f(x)] + (1-\beta)\alpha$. But $\mathbb{E}_{x \sim \mathcal{\mathcal{T}}}$ is just the TPR at detecting samples from $\mathcal{T}$. We can take $\inf_{f: \fpr(f) = \alpha}$ of both sides of the prior equation, which shows that the MIA attack achieving optimal TPR at a fixed FPR $\alpha$ is exactly the optimal hypothesis test for distinguishing samples from $G$ from samples from $D$ with a fixed FPR $\alpha$. Next we note that the optimal hypothesis test in this latter scenario is characterized by the Neyman-Pearson Lemma \cite{neyman}: There exists $\tau_\alpha$ such that $f^*(x) = \textbf{1}\{\frac{G(x)}{D(x)} > \tau_{\alpha}\}$ achieves the maximum achievable TPR at fixed FPR $\alpha$, and so this is the optimal MIA for detecting training samples from $\mathcal{T}$ at a fixed FPR $\alpha$. It remains to be shown that we recover $f^*$ of this form from minimizing the classification error on $\frac{1}{2}D + \frac{1}{2}G$. Under our assumption that we compute a classifier that exactly minimizes the loss on the distribution, it is a standard result that the optimal classifier is the Bayes optimal classifier $B(x) = \frac{G(x)}{G(x) + D(x)}$. Then the result follows from noting that $B(x) > \tau $ is equivalent to $\frac{G(x)}{D(x)} > \frac{1}{\frac{1}{\tau} - 1}$. 
\end{proof}

In practice of course we don't have access to the densities $\gandist(x), \datadistro$. Rather than approximate the likelihood ratio directly as in \cite{refdata}, we observe that as long as (i) the hypothesis class $\mathcal{F}$ is sufficiently rich that $f^*(x)$ is closely approximated by a function in $\mathcal{F}$, standard results on the consistency of the MLE show that $$f^* = \argmin_{f \in \mathcal{F}}-\mathbb{E}_{x \sim \gandist}[\log f(x)] - \mathbb{E}_{x \sim \datadistro}[\log{1-f(x)}],$$ and so we can compute $f^*$ implicitly by minimizing the binary cross-entropy loss classifying samples from $\mixgandist$! Note that we aren't restricted to use the negative log-likelihood, any Bayes consistent loss function \cite{bayesconst} will suffice as a surrogate loss for the classification step; we use negative log-likelihood in our experiments. The assumption $\gandist$ is a simple mixture distribution is a stronger assumption, and Theorem~\ref{thm:opt} is better viewed as showing why our $\detector$ outperforms the random baseline by a significant margin, but is likely far from the information theoretically optimal MIA. We note also that the widely observed phenomenon of partial mode-collapse \cite{modecollapse}, the tendency of GANs in particular to regurgitate a subset of their training data rather than learning the entire distribution, is another motivation for expressing $\gandist$ as a mixture of $\mathcal{T}$ and $\datadistro$.

\subsection{Distance Based Attacks.} \label{onedist} 
The distance-based black-box attacks described in this section are based on the work of ~\cite{chen_gan-leaks_2020}. The intuition behind the attack is that points $\tau$ that are closer to the synthetic sample than to the reference samples are more likely to be classified as being part of the training set. We define $3$ reconstruction losses that we use to define $3$ types of distance-based attacks. Given a distance  metric $\delta$, a test point $\tau \sim \midist$, and samples $X_G \sim \gandist, |X_G| = n$ from the target GAN, we define the generator reconstruction loss $R(\tau |X_G) = \mathop{\arg \min}\limits_{s \in X_G}\delta(s_i, \tau)$. Similarly, given a distance metric $\delta$, a test point $\tau$ and reference samples $X_R \sim \datadistro$,  we define the reference reconstruction loss $R_{ref}(\tau |X_R) = \mathop{\arg \min}\limits_{x \in X_R}\delta(x, \tau)$. Then we can define the relative reconstruction loss, $R_{L}(\tau|X_G,X_R) =  R(\tau|G) -  R_{ref}(\tau|X_R)$. In the \textit {two-way distanced-based} MI attack, we compute  the relative  reconstruction loss, $R_{L}(\tau|G, X_R)$, and predict $\tau \in \xtrain$ if $x > \lambda$ for some threshold $\lambda$. In the \textit {one-way distanced-based} MI attack, we substitute the relative reconstruction loss with the \textit{reconstruction loss relative to the GAN}, $R_{L}(\tau|G)$. 

Given a set $\Delta$  of distance  metrics, for example hamming  and Euclidean distance, we can naturally define a weighted distance-based attack that computes a weighted sum of the reconstruction losses for each distance metric:

\[ \mathcal{W} =\sum\limits_{\delta \in \Delta} \alpha_{\delta} (R^{\delta}(\tau|G) -   R_{ref}^{\delta}(\tau|X_R) ) = \sum\limits_{\delta \in \Delta} \alpha_{\delta} (   R_{L}^{\delta}(\tau|G,X_R) ) \]
The notations $R^{\delta}(\tau|G)$ and $R_{ref}^{\delta}(\tau|X_R)$  emphasize dependence on the distance metric. The weighting parameter $\alpha_{\delta}$  is a measure of how important the associated distance metric is for membership inference,  and could be selected based on domain knowledge. 

In the genomic data setting, we also implement the distance-based attack of  ~\cite{homer_resolving_2008} that is tailored to attacks on genomic data. An improved   ``robust" version of this attack was proposed in ~\cite{robusthomer}. The detailed outline of the robust version of this attack is presented  in Algorithm~\ref{alg:homer} in the Appendix.

\section{Attacks on Tabular Data GANs}
In this section we discuss MIA attack success against a variety of GAN architectures trained on tabular data. We evaluate $7$ attacks across datasets of dimension $d = 805, 5000, 10000$. We first discuss the genomic datasets we used to train our target GANs and subsequently evaluate our attack methods.

\subsection{Experimental Details}
We evaluate our attacks on genomics data from two genomics databases 1000 Genomes (1KG) ~\cite{10002015global}  and dbGaP ~\cite{mailman2007ncbi}). We pre-processed the raw data from dbGaP to Variant Call  Format (VCF) ~\cite{vcf} with PLINK ~\cite{plink}, we used the pre-processed 1KG data provided by ~\cite{yelmen_creating_2021} without any further processing. For both datasets we sub-sampled the features to create three datasets of dimensions $\{805,5000,10000\}$ by selecting regular interval subsets of the columns. For each sub-sampled dataset, we trained two types of target GANs: the vanilla GAN \cite{goodfellow14} using the implementation from \cite{yelmen_creating_2021}, and the Wasserstein GAN with gradient policies (WGAN-GP) \cite{gulrajani_improved_2017}. For the 1KG dataset we train the target GAN on $3000$ samples, and use a further $2008$ samples to train the $\detector$. We evaluate our attacks using another held out sample of $500$ test points, and $500$ randomly sampled points used to train the target GAN. On dbGaP we train the target GAN,  $\detector$, and evaluate on $6500, 5508, 1000$ data points respectively. Table \ref{tab:data-config} in the Appendix summarizes this setup. Figures $7a - 7f$ depict the top $6$ principal components of both the training data and
synthetic samples for each configuration in Table~\ref{tab:data-config}, showing visually that our GAN samples seem to converge in distribution to the training samples. We further verify that our GANs converged and did not overfit to the training data in Appendix~\ref{sub:overfit}. In total, $7$ attack models  (\textit{one-way distance}, \textit{two-way distance}, \textit{robust  homer}, \textit{weighted distance}, \textit{\detector}, \textit{ADIS}, \textit{DOMIAS})  were evaluated on different genomic data configurations. Each of our distance attacks we introduce in Section~\ref{onedist} requires a further choice of distance metric, $\delta$, that  might be  domain-specific and influenced by the nature of the data. We tested our attacks using Hamming distance and Euclidean distance, finding that Hamming distance performs better across all settings, and so we use that as our default distance metric. For the Robust Homer attack, in order to compute the synthetic data mean and reference data means we use all of the reference samples except one held-out sample, and an equal number of samples from the generator. We defer further details on detector architecture and training hyperparameters to Appendix~\ref{appendix:architecure}. Table~\ref{tab:data-distin2} Appendix ~\ref{par:score} lists the $\detector$'s test accuracy on $\mixgandist$, verifying we are successful in training the $\detector$. The  details for ADIS training are discussed in detail in Appendix~\ref{adis}, but at a high-level ADIS uses the same basic architecture as the $\detector$, but augments the input point $x$ with distance-based statistics, and the DOMIAS likelihood statistic. 

\subsection{MIA Success on Genomic GANs}
We now discuss the relative attack success of our methods, reporting full ROC curves plotted on the log-log scale (Figures ~\ref{fig:roc_gen},~\ref{fig:rev3} ), as well as Table~\ref{tab:summary} which reports the TPR at fixed low FPR rates for each attack. For both datasets, we trained average the ROC Curves and TPR/FPR results over $11$ training runs where each time we trained a new target GAN on a new sub-sample of the data. 

\begin{figure}[!htp]
 \centering
 \subfloat[1KG, 805 SNPs, Vanilla GAN]{\includegraphics[width=.3\textwidth]{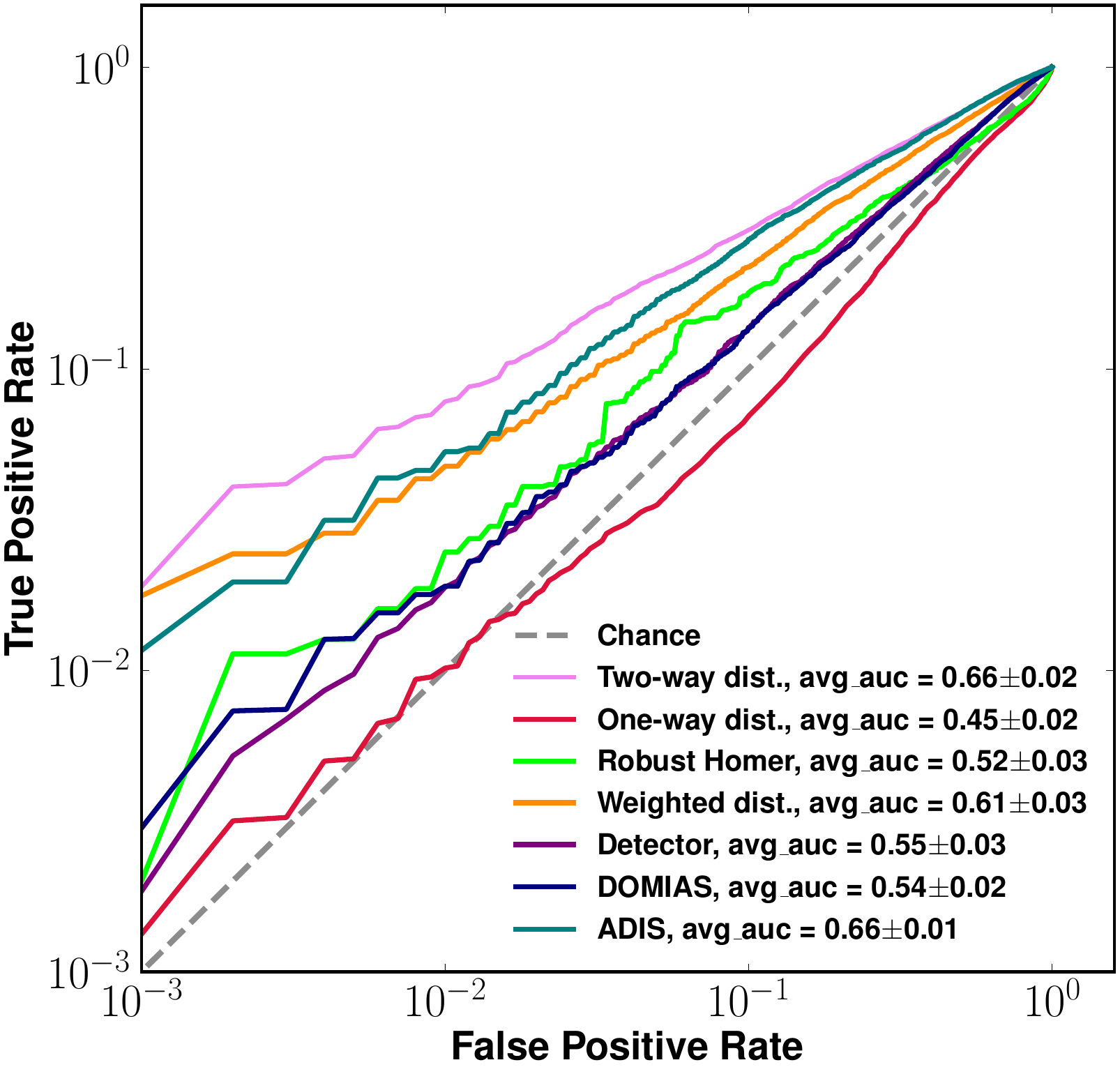}\label{fig:att1}}%
 \subfloat[1KG, 5k SNPs,  Vanilla GAN]{\includegraphics[width=.3\textwidth]{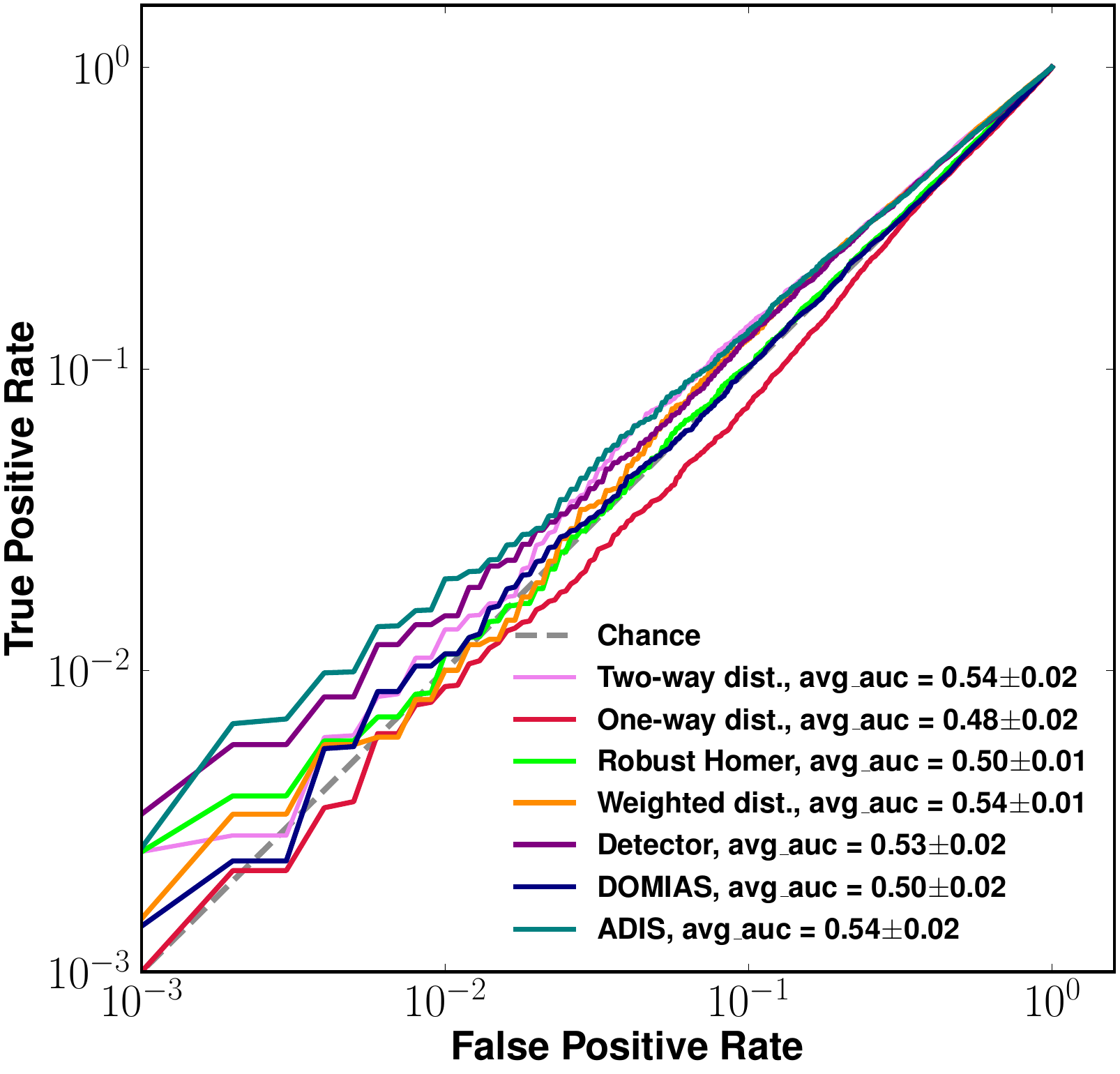}\label{fig:att2}}
 \subfloat[1KG, 10k SNPs, Vanilla GAN]{\includegraphics[width=.3\textwidth]{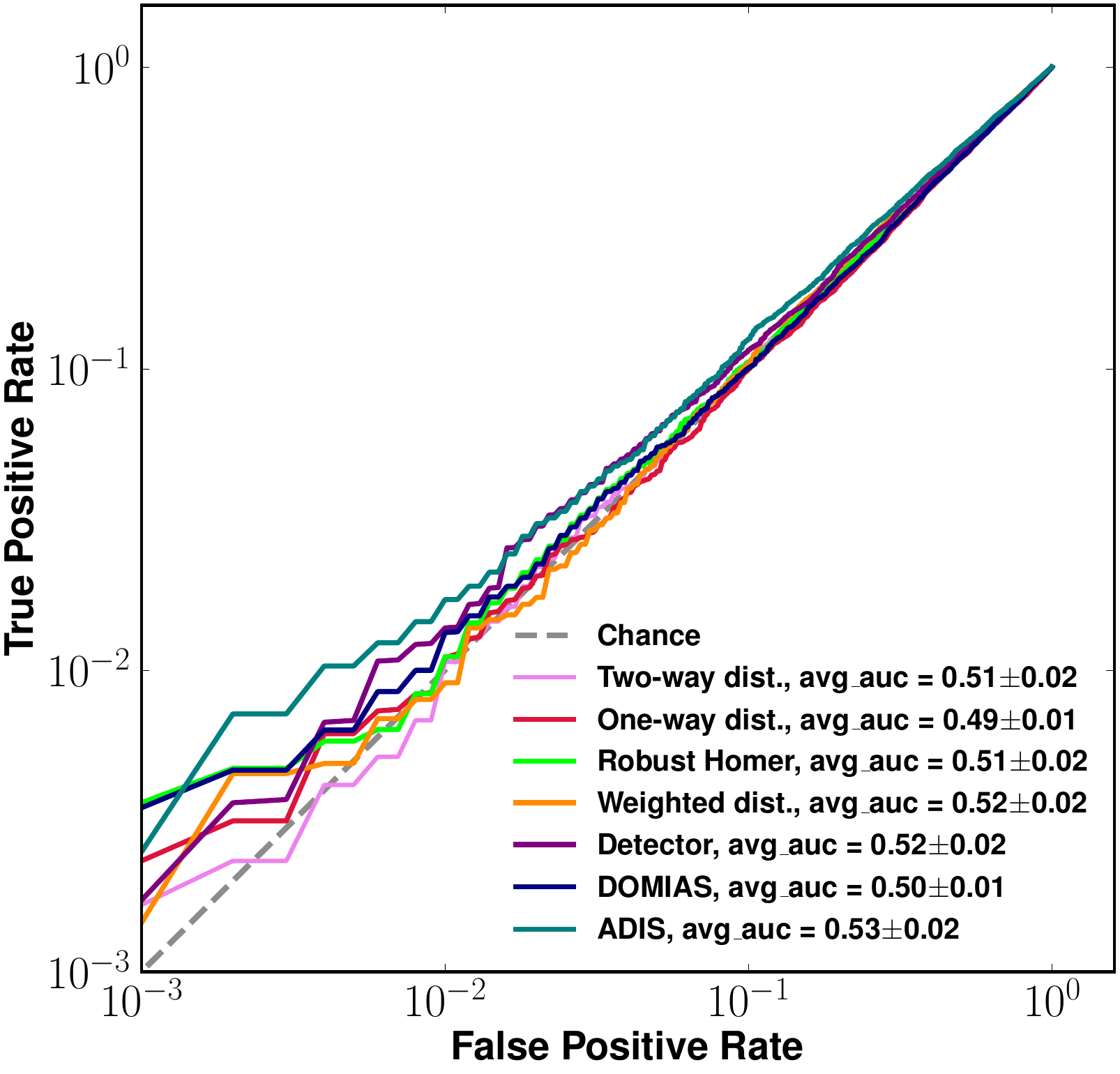}\label{fig:att3}} \\
  \subfloat[1KG, 805 SNPs, WGAN-GP]{\includegraphics[width=.3\textwidth]{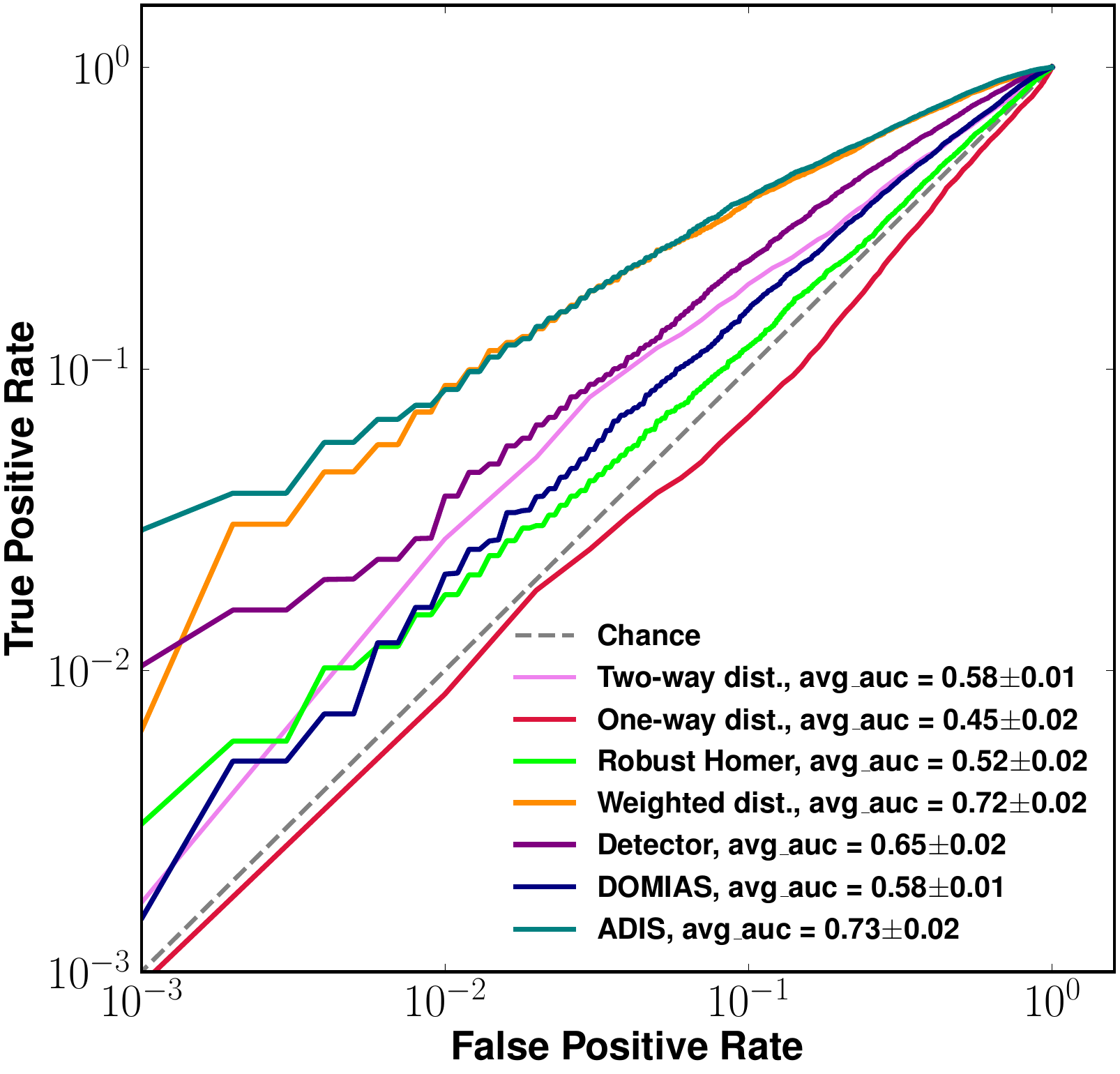}\label{fig:rev1}}%
 \subfloat[1KG, 5k SNPs,  WGAN-GP]{\includegraphics[width=.3\textwidth]{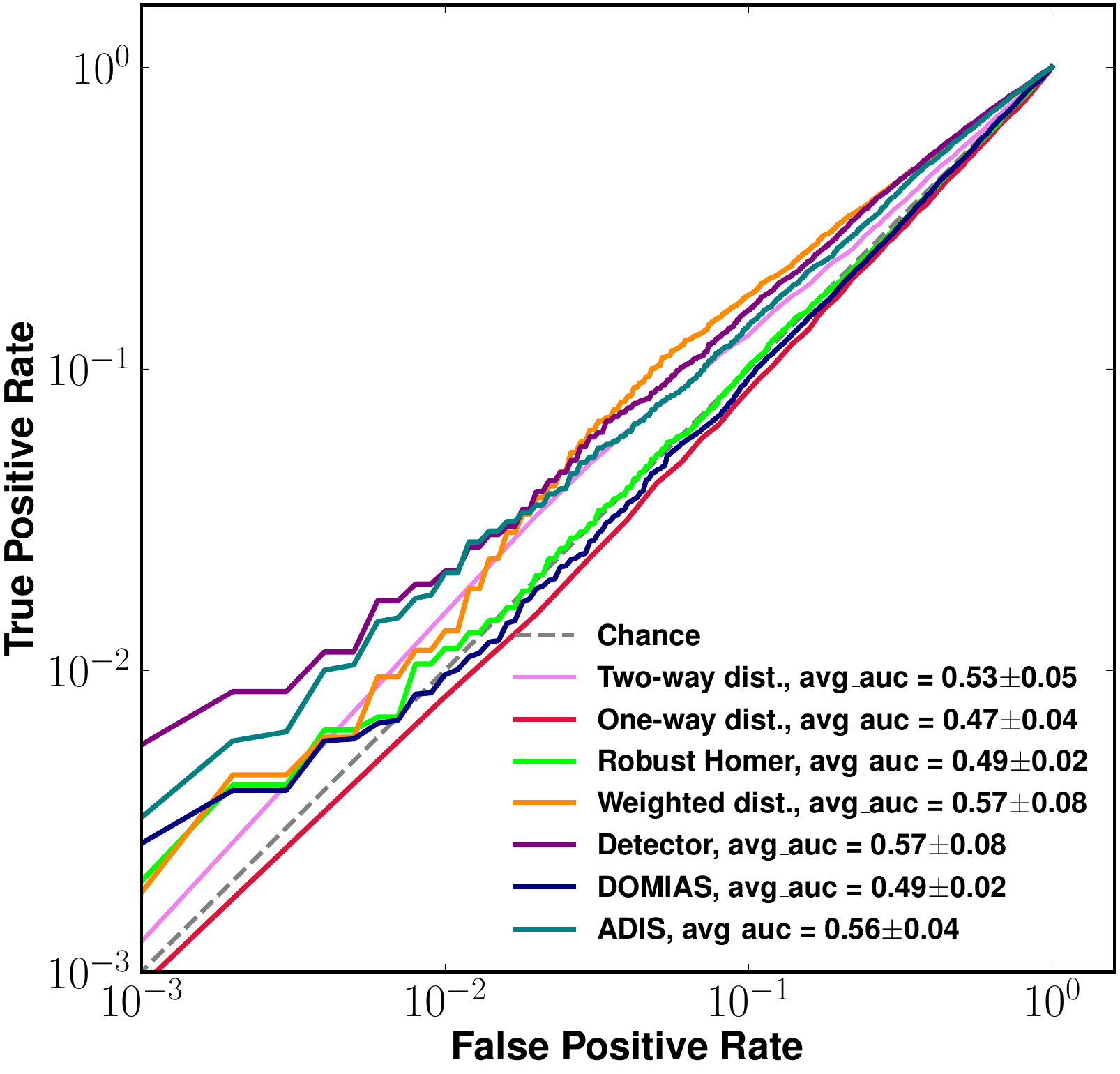}\label{fig:rev2}} 
 \subfloat[1KG, 10k SNPs, WGAN-GP]{\includegraphics[width=.3\textwidth]{ 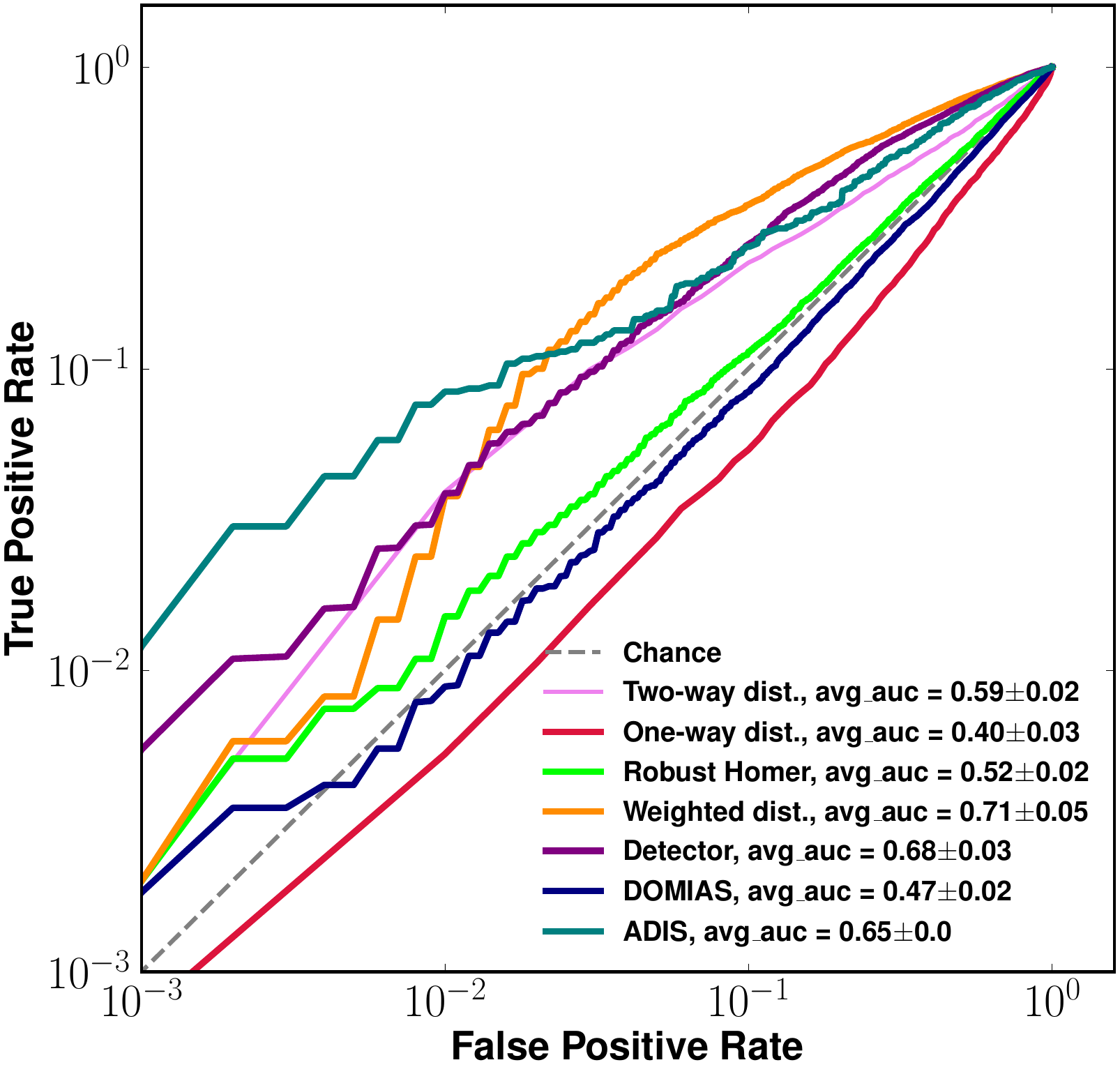}\label{fig:rev3}}\\
 \subfloat[dbGaP, 805 SNPs.,Vanilla]{\includegraphics[width=.3\textwidth]{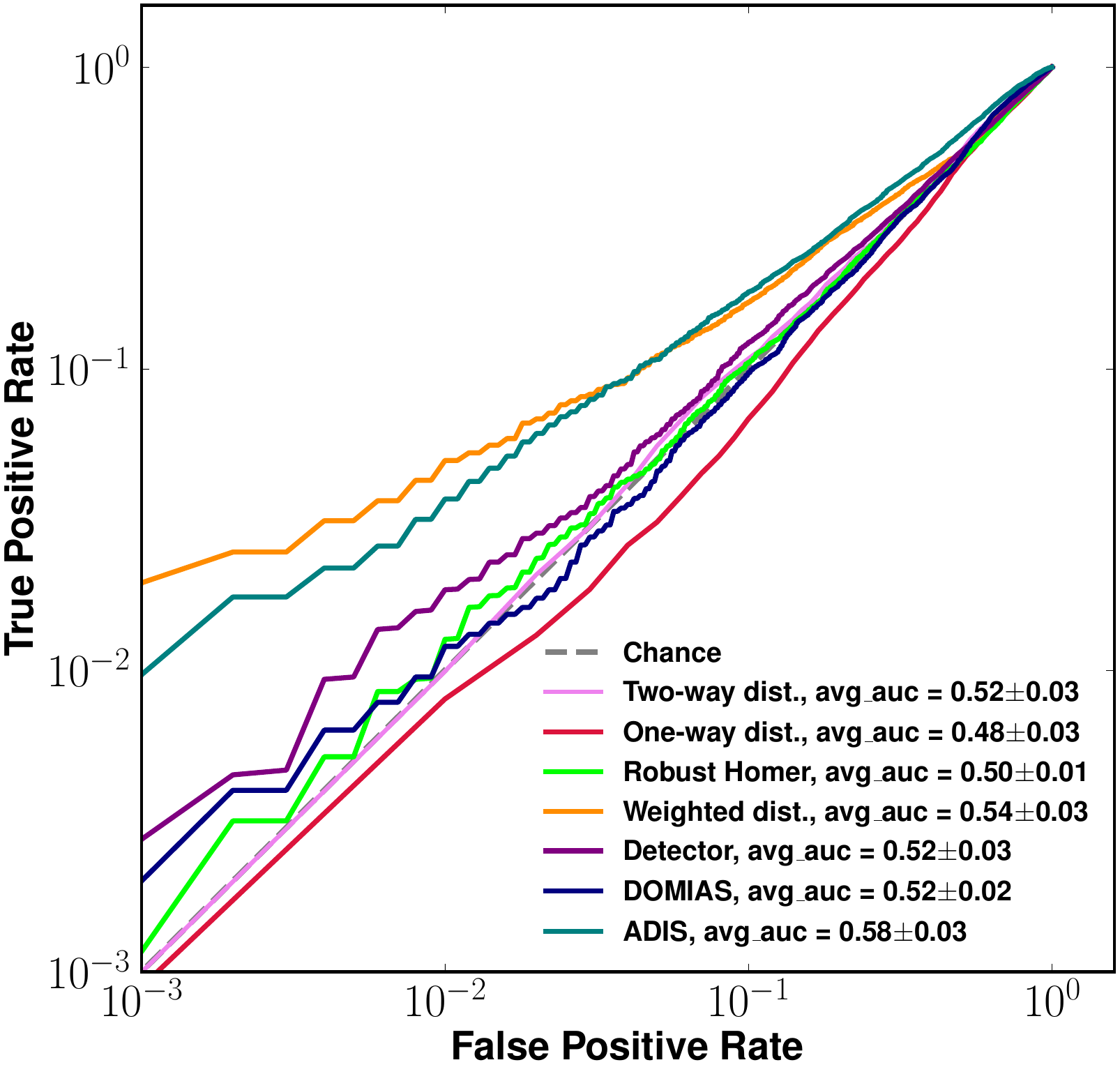}\label{fig:vanil_db1}}%
 \subfloat[dbGaP, 5k SNPs.,Vanilla]{\includegraphics[width=.3\textwidth]{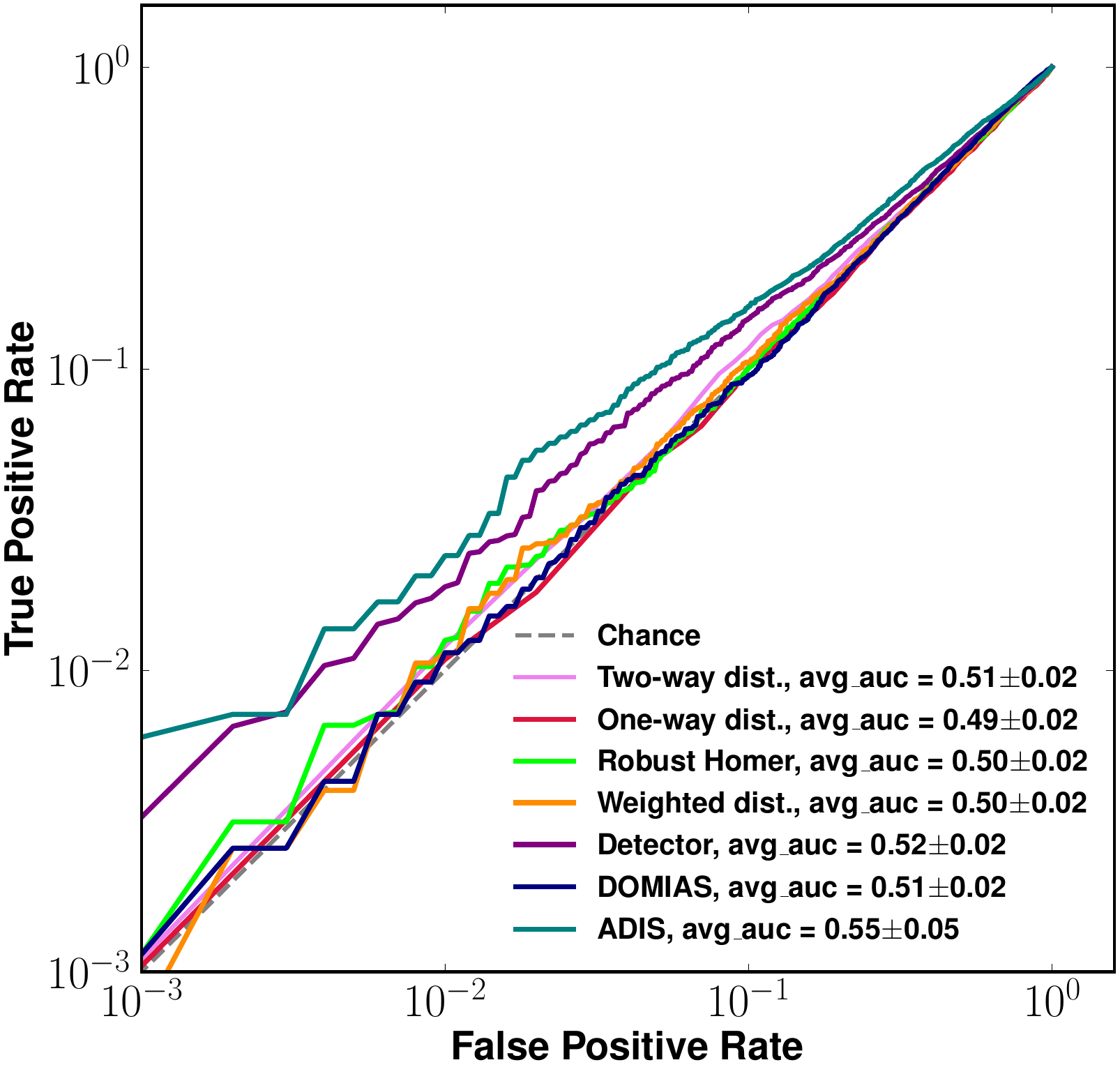}\label{fig:vanil_db2}}%
 \subfloat[dbGaP, 10k SNPs.,Vanilla]{\includegraphics[width=.3\textwidth]{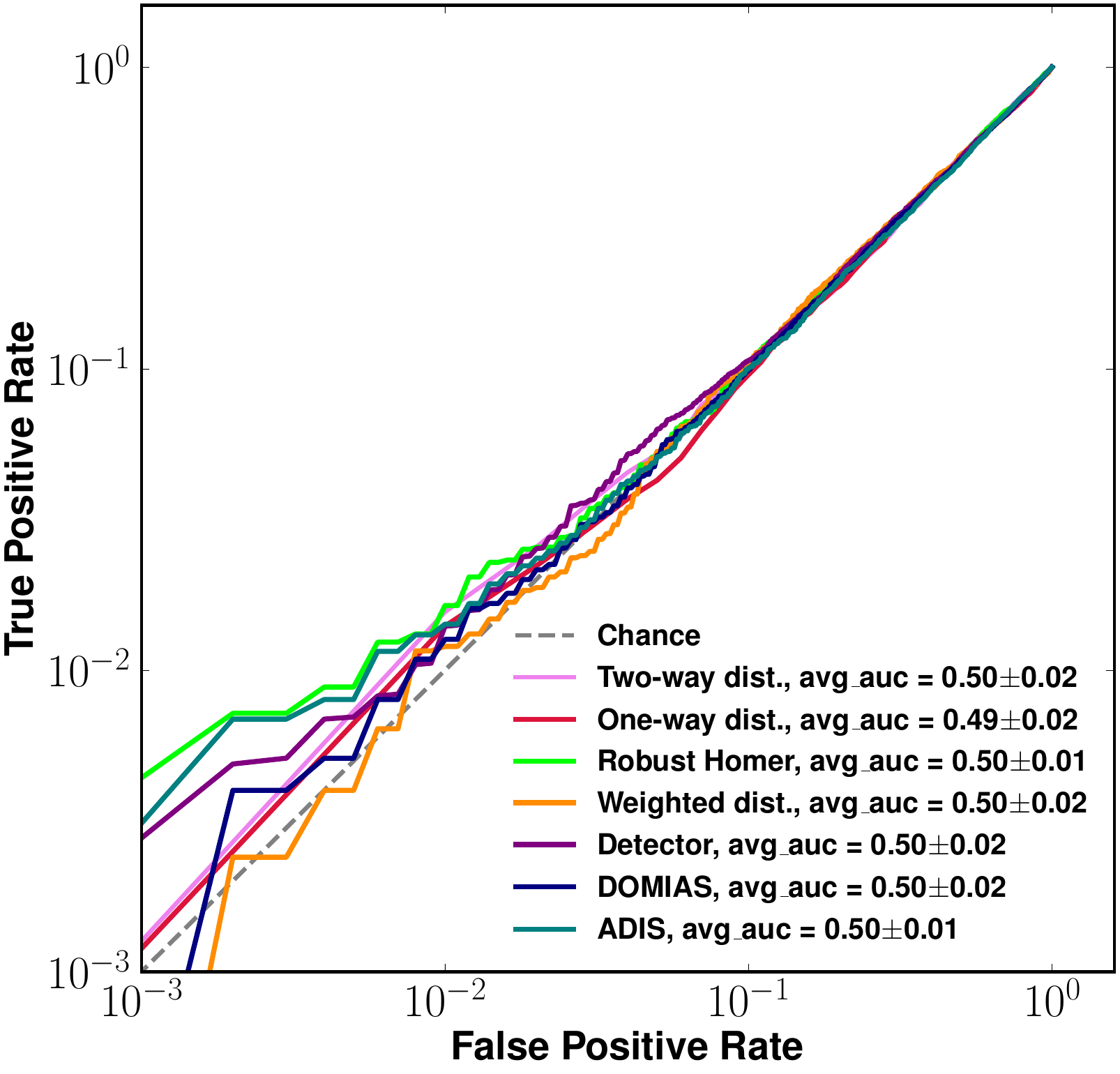}\label{fig:vanil_db3}}\\
 \subfloat[dbGaP, 805 SNPs.,WGAN-GP]{\includegraphics[width=.3\textwidth]{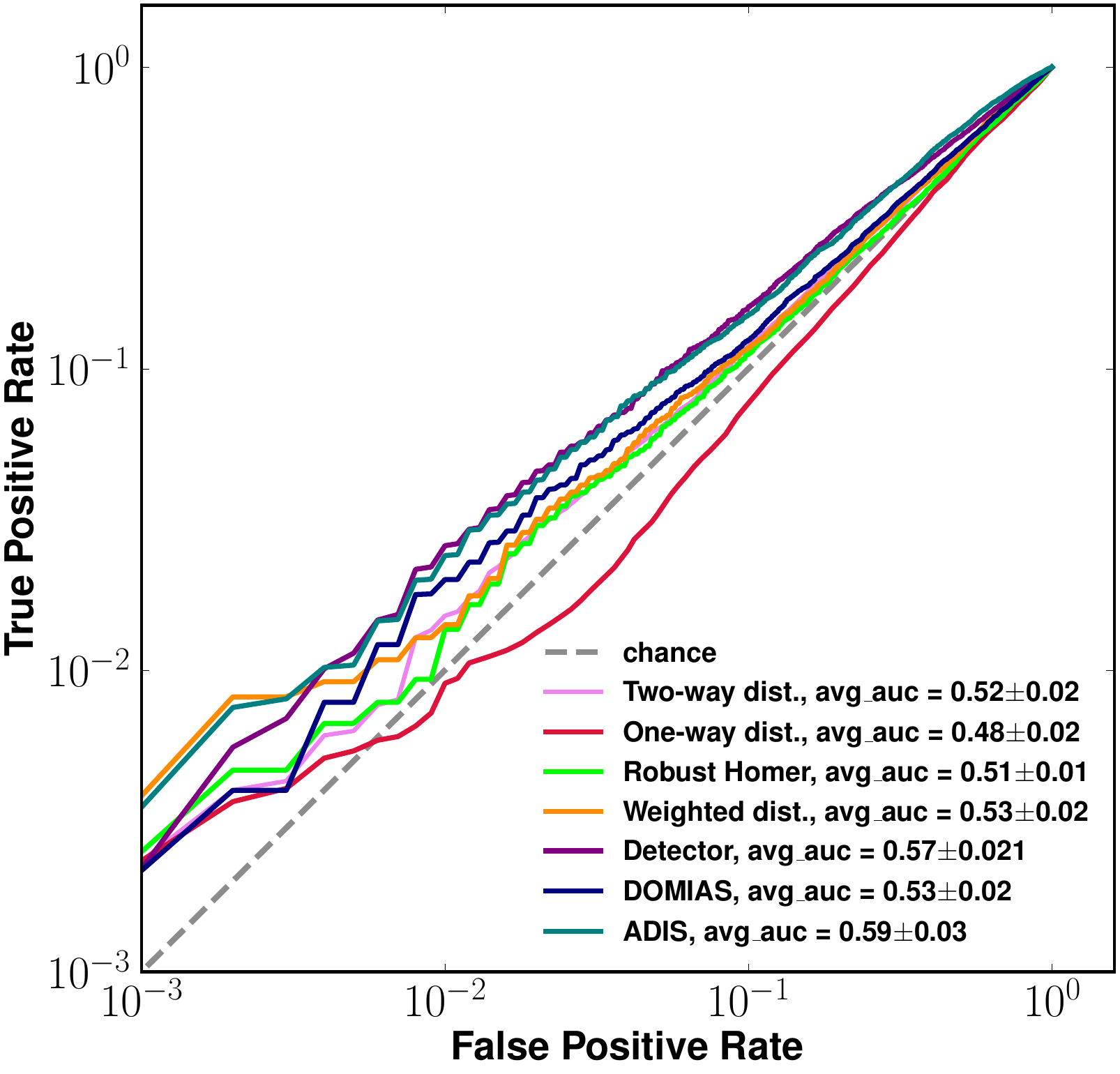}\label{fig:att4}}%
 \subfloat[dbGaP, 5k SNPs.,WGAN-GP]{\includegraphics[width=.3\textwidth]{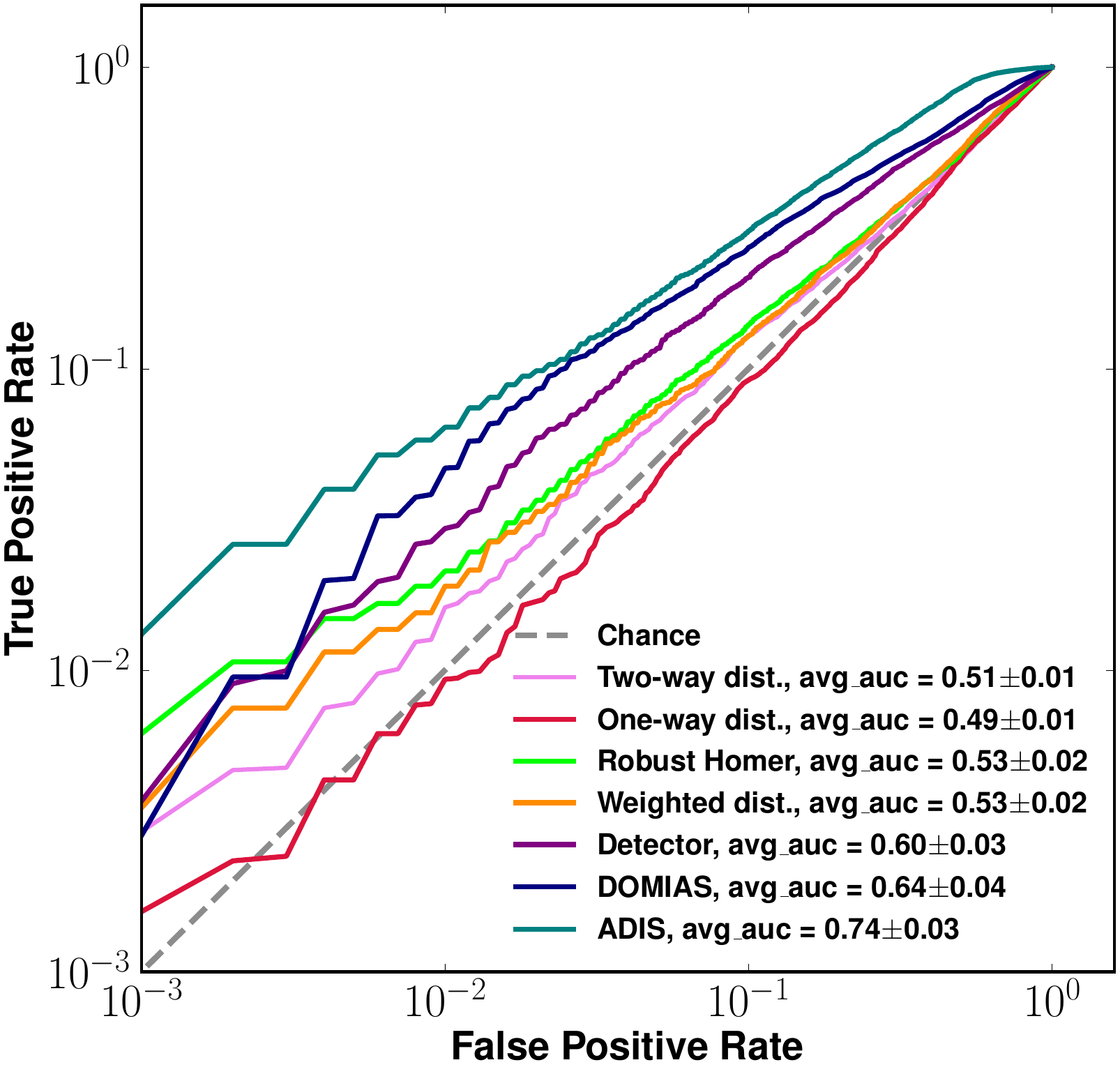}\label{fig:att5}}
 \subfloat[dbGaP, 10k SNPs.,WGAN-GP]{\includegraphics[width=.3\textwidth]{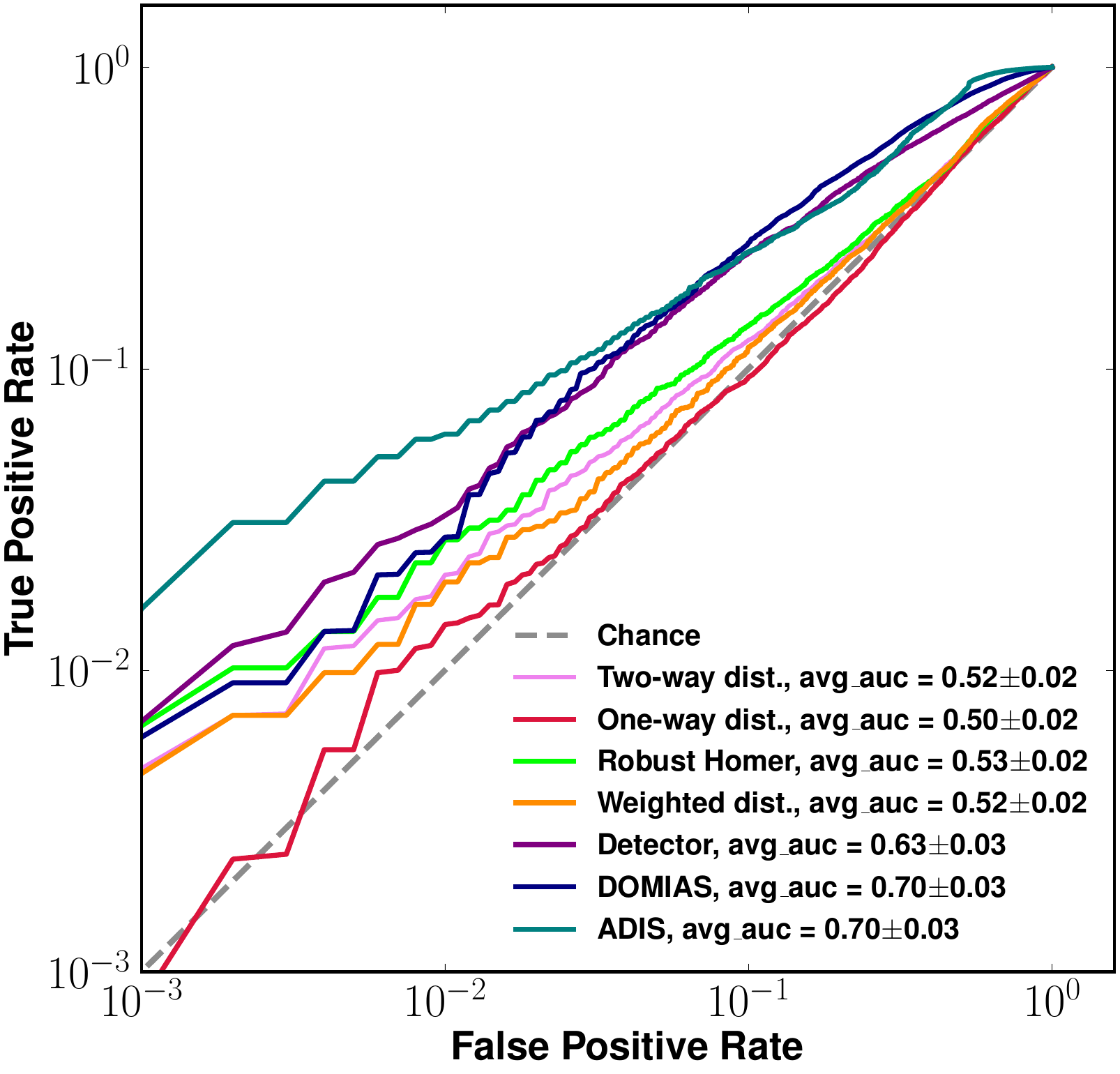}\label{fig:att6}}%
 \caption{ log-log ROC Curves for Vanilla GAN and WGAN-GP trained on 1KG and dbGaP.}
\label{fig:roc_gen}
\end{figure}

Inspecting the results on 1KG in Figure~~\ref{fig:roc_gen}, we see that when the target GAN is the Vanilla GAN, at $805$ SNPs the two-way distance attack achieves both the highest TPRs at low FPRs, and the highest AUC overall at $.66$, with all the attacks except the one-way distance attack outperforming the random baseline (curves above the diagonal) by a significant margin. As the number of SNPs increases to $5K$ and then $10K$, we see that the two-way distance attack performs worse than ADIS and the $\detector$, even performing worse than random guessing at $10K$. As the number of SNPs increases, the relative performance of the Robust Homer Attack increases, which we also observe in the WGAN-GP results, and on the dbGaP dataset. When the target GAN is a WGAN-GP, detector-based attacks outperform distance-based attacks or DOMIAS at every SNP dimension, which we also observe for the dbGaP dataset. While the trends in terms of relative performance of different attack methods are consistent across datasets, the actual levels of privacy leakage differ. Inspecting the dbGaP results in Figure ~\ref{fig:rev3} we see lower AUCs across the board, although we still observe high TPRs at low FPRs. 

More generally, while the AUC results for the attack methods indicate moderately low privacy leakage relative to white-box attacks against the discriminator \cite{LOGAN2,chen_gan-leaks_2020,Mi-montecarlo} or those reported against diffusion models \cite{extract_diff}, for the more meaningful metric of FPR at low fixed TPRs \cite{firstprinciple}, Table~\ref{tab:summary} shows that for fixed FPRs ADIS in particular achieves TPRs that are as much as $10$x the random baseline (FPR = TPR). ADIS works especially well against WGAN-GP trained on dbGaP, with improved attack success for larger dimension. 

In summary, while distance-based attacks are occassionally competitive with the detector-based attacks, the $\detector$s  have more robust attack performance as we vary the data dimension, target GAN architecture, and dataset. While these factors impact the overall privacy leakage, they have little effect on the relative ordering of the attacks. While overall AUC is low, when we consider TPR at a fixed FPR, ADIS achieves attack results that are in some cases an order of magnitude better than the random baseline, evidence of serious privacy leakage. Moreover, the fact that ADIS consistently outperforms the $\detector$ means that augmenting the feature space with distance-based features improved attack success relative to using the input point alone. This could be due to the fact that given our relatively small datasets, starting from the distance-based metrics enabled the algorithm to learn features more efficiently.

\section{Attacks on Image GANs}

\begin{figure}[!htp]%
 \centering
 \subfloat[Progressive GAN, $\cifar$]{\includegraphics[width=.3\textwidth]{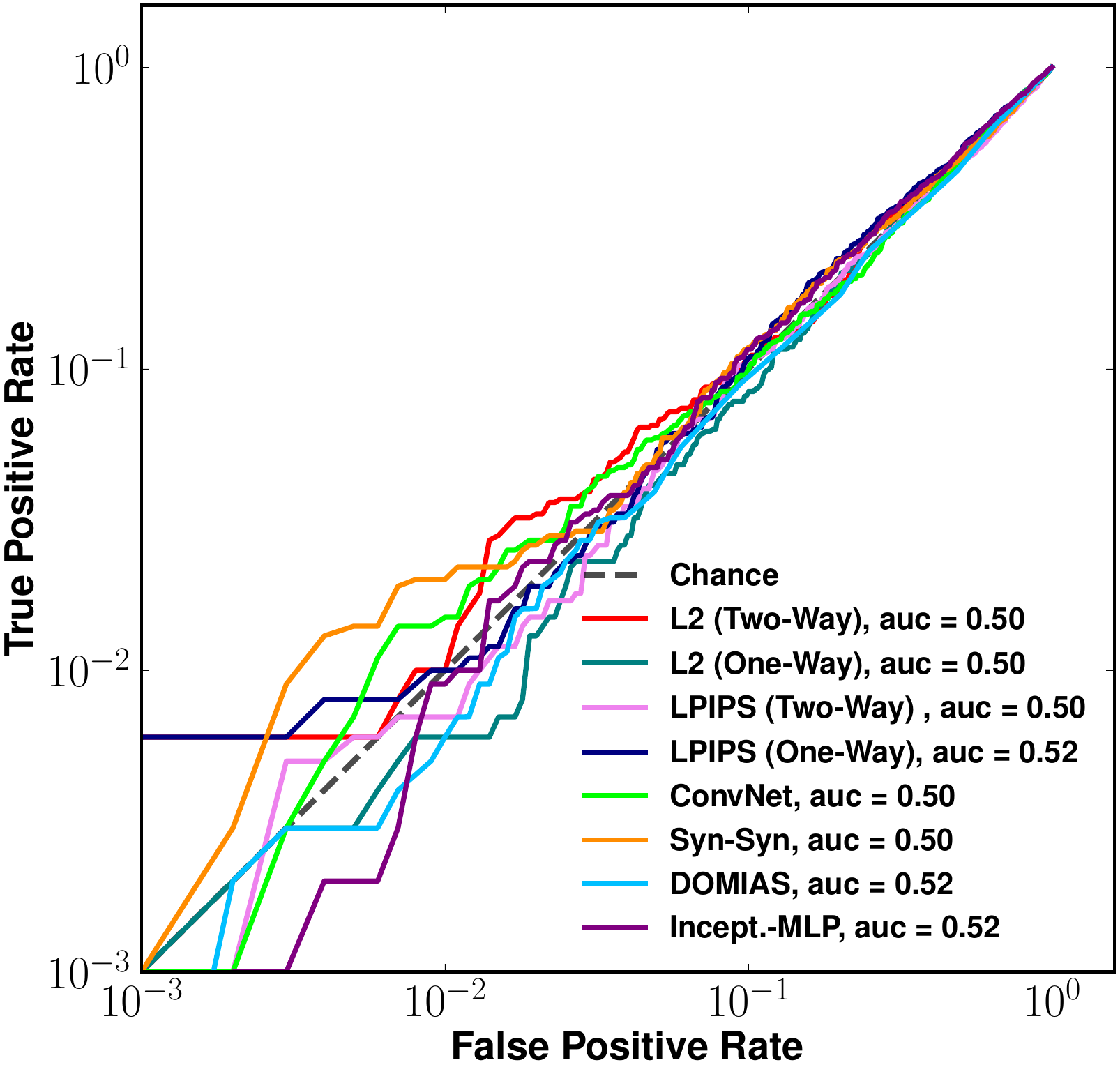}\label{fig:attIm3}}%
 \hspace{1cm}
 \subfloat[ContraGAN, $\cifar$]{\includegraphics[width=.3\textwidth]{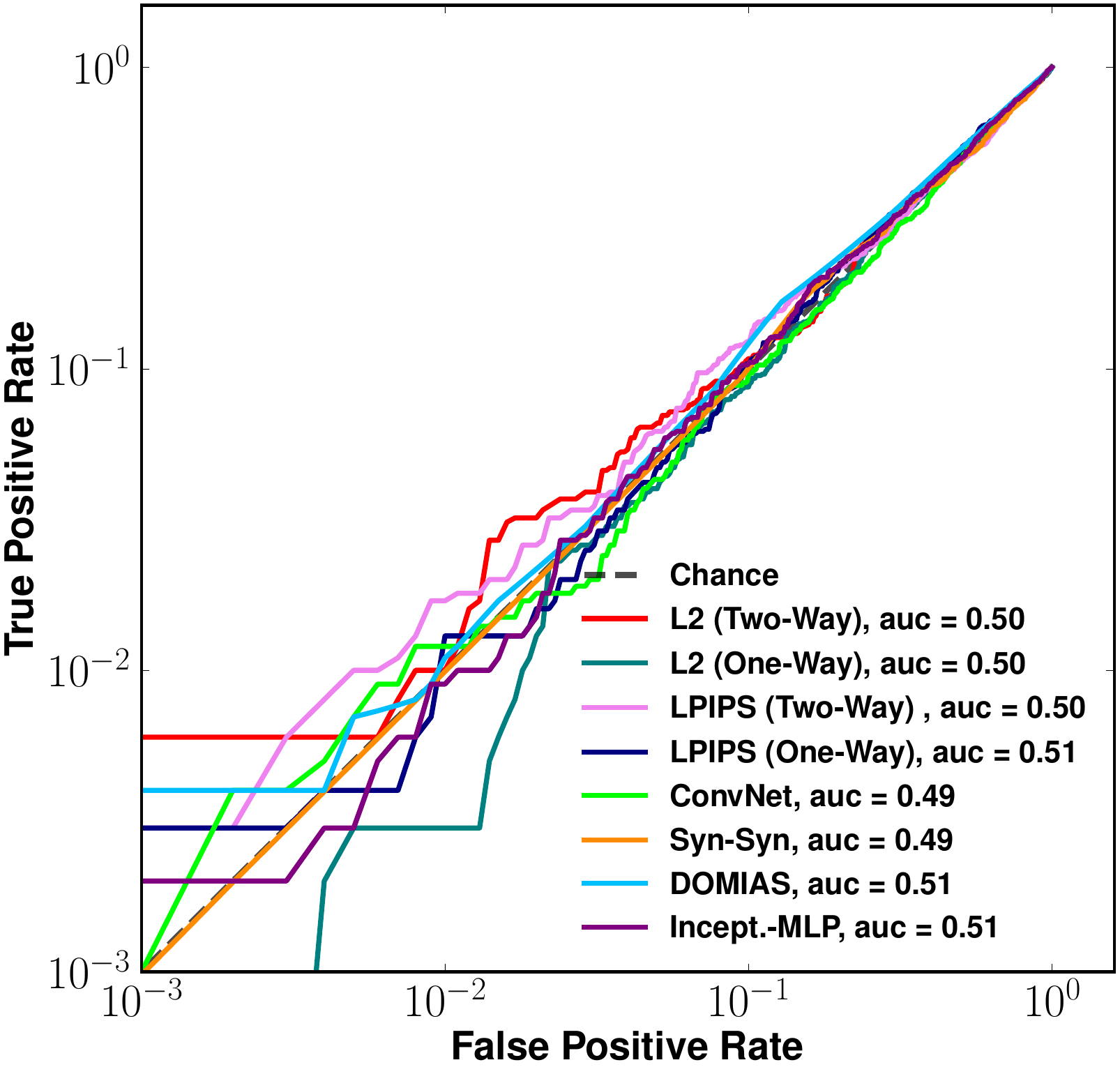}} \\
 \subfloat[BigGAN, $\cifar$]{\includegraphics[width=.3\textwidth]{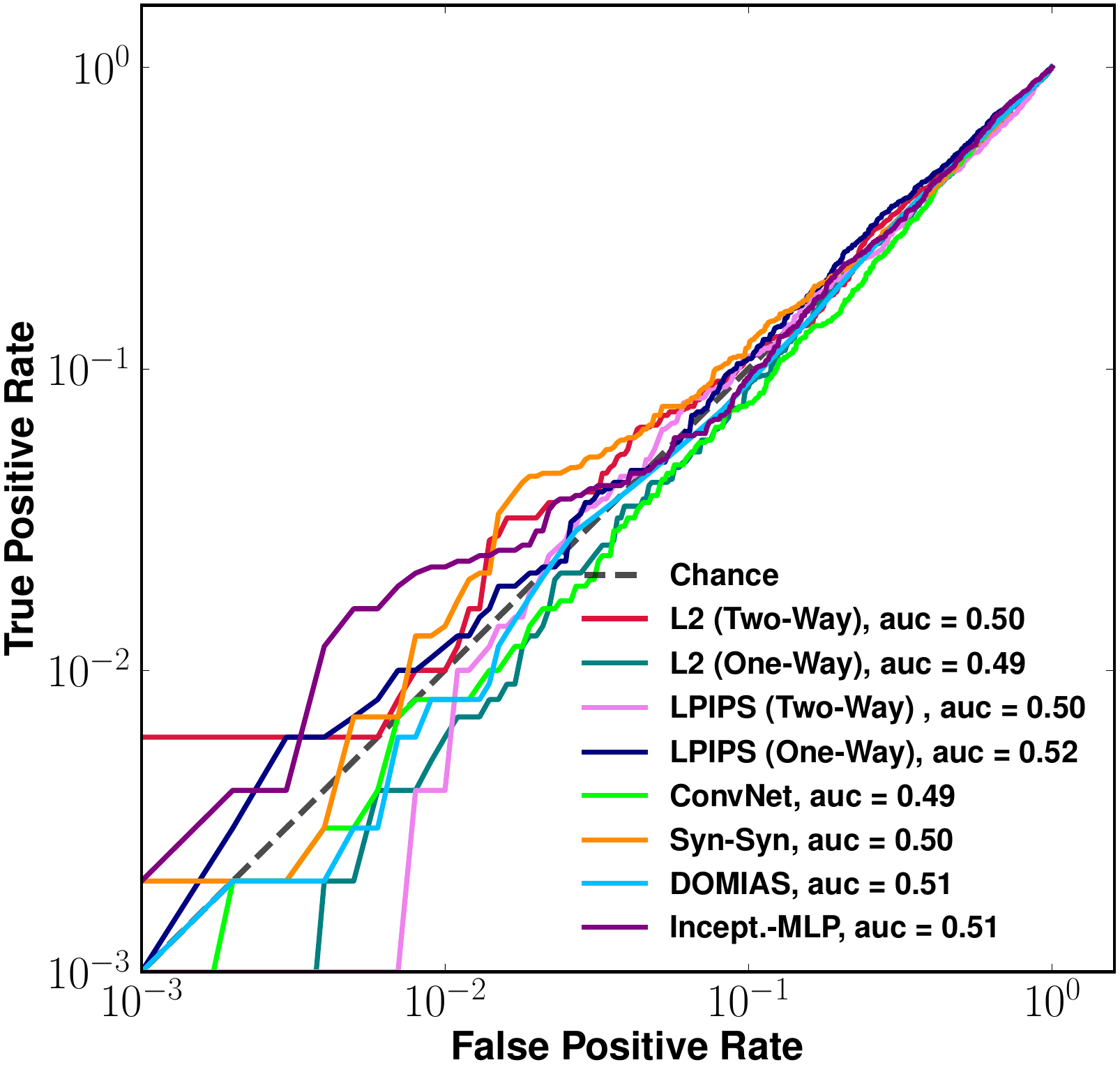}\label{fig:attIm1}}%
 \hspace{1cm}
 \subfloat[DCGAN, $\cifar$]{\includegraphics[width=.3\textwidth]
{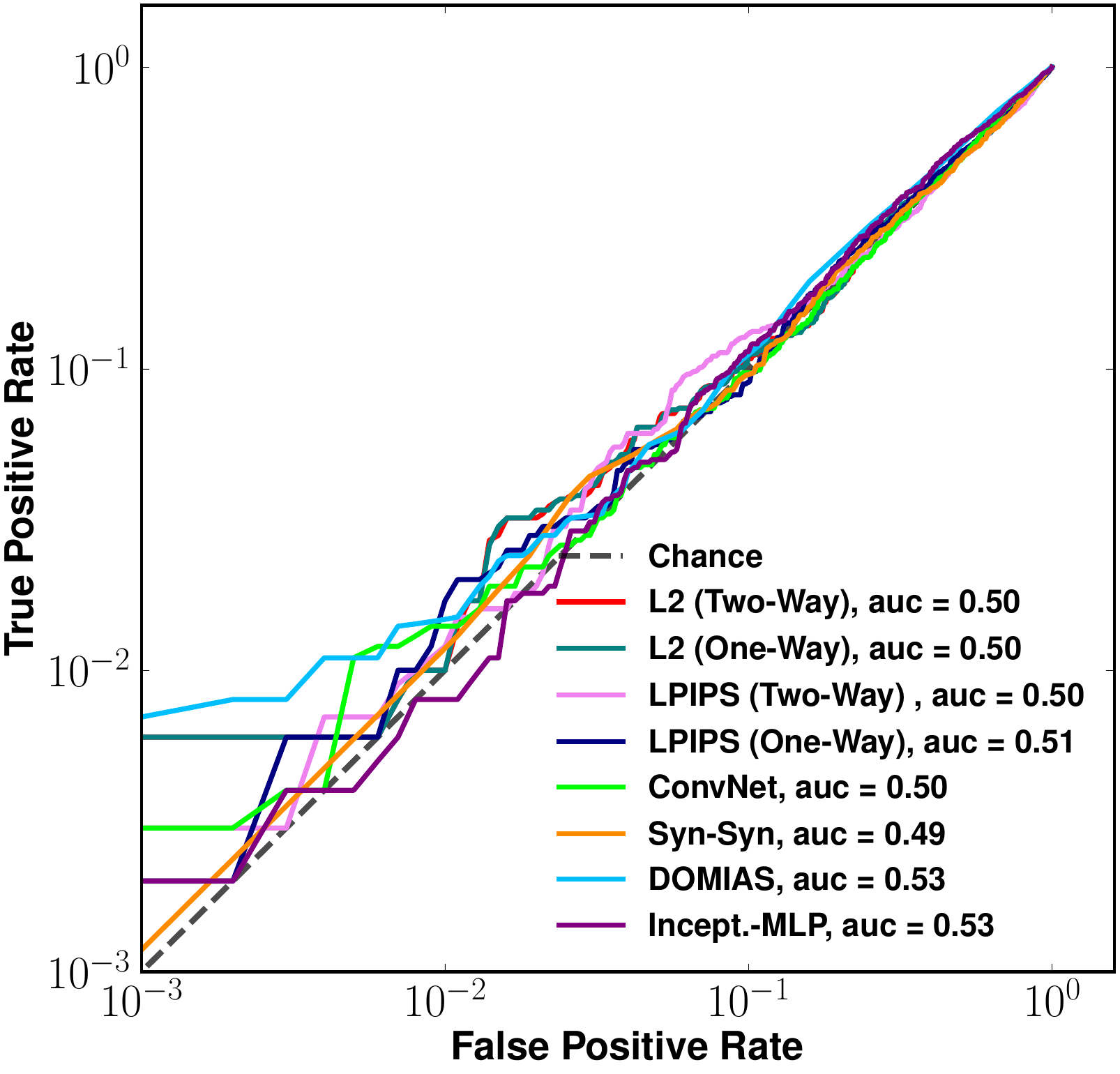}\label{fig:attIm2}}
 \caption{log-log ROC curves for detector-based MIAs against different GAN architectures trained on $\cifar$. While the AUCs are lower than the genomic settings, the different variants of the detector-based attacks work very well at low FPRs.}
 \label{fig:imdis2}
\end{figure}

In this Section, we evaluate our attacks against $4$ GAN architectures trained on images from the $\cifar$ dataset. We observe that relative to high dimensional tabular datasets, image GANs exhibit less privacy leakage on average, with AUCs that are barely above the random baseline. However, for the most meaningful metric, TPR at low FPRs, our detector-based attacks, and in some cases the distance-based attacks, are again able to achieve TPRs that are $2-6 \times$ higher than the random baseline. 

\subsection{Experimental Details}
We focused on GANs trained on the $\cifar$ image dataset~\cite{cifar}. $\cifar$  consists of $60000$ images across $10$ object categories, and so when sub-sampling the dataset for training the target GAN, $\detector$s, and evaluation, we performed stratified sampling to ensure balance across categories. The test sample size was set to $2000$ such that we have $200$ objects from each category, with $50000$ images used to train the target GAN, and $9000$ used to train the $\detector$s. 

Four state-of-the-art GAN variants were trained: large-scale GANs for high fidelity image synthesis ~\cite{biggan}, Conditional GANs with Projective Discriminator (PDGAN) ~\cite{cgan}, Deep Convolutional Generative Adversarial Networks (DCGAN) ~\cite{dcgan}, and Contrastive Learning for Conditional  Image Synthesis ~\cite{contragan}. We trained the GANs using the implementations in the Pytorch StudioGAN library \cite{studio}.

\sn{For the image GANs we change the distance metrics to $\ell_2$ and LPIPS, one or two sentences about why we did this, and what LPIPS is}\luk{the details about the choice of these 2 distance metrics are in the appendix. Appendix \ref{distance_image_metrics}. It was moved there becuase we had all distance-based attack moved to appendix.}
We observed that training the $\detector$ on image data proved a more delicate process than distinguishing genomic data -- as a result we implemented $3$ $\detector$ architectures. \sn{for each of these variants put in the parenthesis the same name we refer to them in the plot as} The first variant trains a CNN-based $\detector$ ~\cite{imagenet, underatsningCNN}. In the second variant we extracted features from the training images using a pre-trained model (Inception model architecture) ~\cite{inception} and fed this to an MLP pipeline. In the final variant, we trained a variational auto-encoder (VAE) ~\cite{vae} on the $9000$ reference data images, and then trained a CNN classifier ~\cite{imagenet} on synthetic samples from the VAE (label $0$) and synthetic samples from the target GAN (label $1$). This idea of training a second generative model on the reference data was proposed in \cite{LOGAN2}. Further details on $\detector$ architectures, training parameters, and convergence plots are deferred to Section~\ref{sec:gen_convergence} in the Appendix. \sn{Reference the right appendix section here}

\subsection{MIA Attack Success}
Recall that for the image GANs we have $3$ variants of the detector attack, DOMIAS, and two distance measures ($\ell_2$, LPIPS \cite{lpips}) were utilized for the distance-based attack. Each distance measure contributes $2$ attack types: one-way and two-way, and so for the full results in the Appendix we have a total of $8$ attacks. We summarize the results in Figure~\ref{fig:imdis2} and Table~\ref{tab:summary}  We see that at low FPRs $< 10^{-2}$, the $\detector$ that achieves the highest TPR varies based on the target GAN architecture. For all of the target GANs Incept-MLP has the highest AUC of the detector-based methods, but performs poorly at low FPRs, albeit on BigGAN, where its AUC is about the same as the other methods.  DOMIAS performs better on a relative basis on the image GANs, with the best performance in terms of AUC and low FPRs on DCGAN and second best performance at low FPRs on ContraGAN. However, when the target GAN was Progressive GAN or BigGAn, DOMIAS performed worse than random guessing at low FPRs, and so its performance overall seems inconsistent. This is an interesting direction for future investigation.

\sn{Discuss the experimental results!}
\section{Conclusion}
Our work provides the most thorough existing empirical analysis of attacks against GANs in the black-box setting. We prove theoretical results that build intuition for why a
network that can detect samples generated from the GAN can also detect samples from the training set (Theorem~\ref{thm:opt}), a result that has applications across all privacy attacks on ML models, not just GANs, and conduct an extremely thorough empirical evaluation of detector-based attacks, comparing their performance to existing black-box attacks across tabular and image domains. This work also raises several interesting directions for future research: although we expose significant privacy leakage from GANs, particularly at low FPRs, relative to other generative models explored in prior work, for example VAEs \cite{LOGAN} or diffusion models \cite{extract_diff}, black-box access to the generator seems much more private. Is this actually the case, or is it simply a matter of developing better attack methods? If so, can we prove theoretically why GANs appear to be more privacy-preserving than other generative models? 

% \section*{Acknowledgments}
% This was was supported in part by......

%Bibliography
\bibliographystyle{unsrt}  
\bibliography{references}  

\clearpage
\appendix
\section{Lemmas and Proofs}
We include the lemmas and proofs from the main paper in this section.
\begin{lemma}
\label{lem:was}
$$
\text{TPR}_{\midist}(f) \leq \text{TPR}_{\mixgandist}(f) + \frac{\text{Lip}(f)}{2}W^{1}(\gandist, \mathcal{T}), 
$$
\end{lemma}
\begin{proof}
$\text{TPR}_{\midist}(f) = \frac{1}{2} \mathbb{E}_{\mathcal{T}}[f(x)] = \frac{1}{2}(\mathbb{E}_{\gandist}[f(x)] + \mathbb{E}_{\mathcal{T}}[f(x)] - \mathbb{E}_{\gandist}[f(x)]) = \text{TPR}_{\mixgandist}(f) + \frac{1}{2}(\mathbb{E}_{\mathcal{T}}[f(x)] - \mathbb{E}_{\gandist}[f(x)]) \leq \text{TPR}_{\mixgandist}(f) + \frac{1}{2}|\mathbb{E}_{\mathcal{T}}[f(x)] - \mathbb{E}_{\gandist}[f(x)])| \leq \text{TPR}_{\mixgandist}(f) + \frac{1}{2\text{Lip}(f)}\sup_{g: \text{Lip}(g) = 1}|\mathbb{E}_{\mathcal{T}}[f(x)] - \mathbb{E}_{\gandist}[f(x)])| = \text{TPR}_{\mixgandist}(f) + \frac{1}{2\text{Lip}(f)}W^{1}(\gandist, \mathcal{T})$. 
\end{proof}

\section{Attacks}
We discuss the details of the different attack types in this section.

\subsection{Robust Homer Attack}
The intuition behind the attack lies in Line $7$ of Algorithm~\ref{alg:homer}. Recall  $X_G$ denotes synthetic samples from the GAN, $X_R$ denotes reference data sampled from $\refdistro$, and $\tau \sim \midist$ denotes a test point. Let $\mu_g = \frac{1}{n}\sum\limits_{i=0}^n  x_{gi} \quad  \text{where } \quad x_{g0},x_{g1},  \dots x_{gn} \in X_G \quad  \text{and} \quad \mu_g \in {[0,1]}^d $. Similarly, let $  \mu_r= \frac{1}{m}\sum\limits_{i=\mathbf{1}}^m  x_i \quad  \text{where } \quad x_1, x_2 \dots x_m \in X_R \quad  \text{and} \quad \mu_r \in {[0,1]}^d $. Then sample $x \sim \datadistro$, and compute:
\[  \langle  \tau - x ,  \mu_g - \mu_r  \rangle  =  \langle  \tau ,  \mu_g - \mu_r  \rangle  - \langle  x ,  \mu_g - \mu_r  \rangle = \underbrace{\big[\langle \tau, \mu_g \rangle- \langle x, \mu_g \rangle \big] }_{(1)}  +  \underbrace{\big[\langle x,\mu_r , \rangle -\langle \tau,\mu_r , \rangle \big]}_{(2)} \]
 Part $(1)$ above checks if  $\tau $ is more correlated with $\mu_g$ than a random sample $x$ from $\datadistro$, while $(2)$ checks if $x$ is more correlated with $\mu_r$ than $\tau$. Observe that this is similar to the intuition behind the distance-based attack, but here we compute  an average \textit {similarity measure} to $\mu_g$ and $\mu_r$, rather than a distance to the closest point. \sn{Can salil make sure this description is correct}
 
\begin{algorithm}[!htbp]
\caption{Robust Homer Attack}
\label{alg:homer}
\begin{algorithmic}[1]
\Require $( X_R \in \{0,1\}^d, \gan, x \sim \midist)$
%\Require $y = x^n$
\State $  \mu_r = \frac{1}{m}\sum\limits_{i=1}^m  x_i \quad  \text{where } \quad x_1, x_2 \dots x_m \in X
_R$
\State $\alpha \gets \frac{1}{\sqrt{m}} + \epsilon, \quad \eta \gets 2 \alpha $ \Comment{$\epsilon \geq 0$}
\State $ X_G \sim \gandist $ \Comment{$ X_G\in {\{0,1\}}^d $}

\State $ \mu_g = \frac{1}{n}\sum\limits_{i=0}^n  x_{gi} \quad  \text{where } \quad x_{g0},x_{g1},  \dots x_{gn}\in X_G \quad  \text{and} \quad \mu_g \in {[0,1]}^d$
\State Let ${\left \lfloor{\mu_g - \mu_r}\right \rceil}_\eta \in {[-\eta , +\eta]}^d  $ \Comment{entry-wise truncation of $ \mu_g-\mu_r \, $ to $[-\eta , \eta]$   }
\For{ $\tau \sim \midist$ } 

\State $\rho \gets \langle  \tau - x ,  {\left \lfloor{\mu_g - \mu_r}\right \rceil}_\eta \rangle$
\If{$\rho >  \kappa$} \Comment{$\kappa$ is an hyperparameter}
                \State $\tau $ is in training set for $\gan$
\Else
      \State $\tau$ is not in training set for 
      $\gan$
\EndIf
\EndFor
\end{algorithmic}
\end{algorithm}

\subsection{Distance Attacks on Image GAN}\label{distance_image_metrics}
The choice of  metric for the distance-based attack on image GANs requires some careful consideration. Though $\ell_2$ distance was suggested in ~\cite{chen_gan-leaks_2020}, ~\cite{metrics} notes that it can't capture joint image statistics since it uses point-wise differences to measure image similarity ~\cite{metrics}. More recently, with the widespread adoption of deep neural networks, learning-based metrics have enjoyed wide popularity and acceptance ~\cite{deepsim,netsim,alexmetrics}. 
 These metrics leverage pre-trained deep neural networks to extract features and use these features as the basis for metric computation ~\cite{metrics}. \texttt{LPIPS} ~\cite{lpips} (Learned Perceptual Image Patch Similarity) is the most common of such learning-based metrics and was also proposed in ~\cite{chen_gan-leaks_2020} for carrying out distance-based attacks. Our distance-based attacks are implemented using both $\ell_2$ distance and \texttt{LPIPS}.

\subsection{Detector Attack}
The target GAN configurations  for the genomic data setting are shown in Table ~\ref{tab:data-config}.
\begin{table}[H]
\centering
\setlength{\extrarowheight}{2.5pt} % Adjust cell height
\caption{Target GAN configurations for genomic Data }
\begin{tabular}{ ||c|c|c|c|c|c|| } 
\hline
Data & SNPs Dim. & Train Data Size & Ref. Data Size & Test Size & GAN Variant \\ 
\hline \hline
\multirow{3}{*}{\centering 1000 Genome} & 805 & 3000 & 2008 & 1000 & vanilla \& WGAN-GP \\ \cline{2-6}
& 5000 & 3000 & 2008 & 1000 & vanilla \& WGAN-GP \\ \cline{2-6}
& 1000 & 3000 & 2008 & 1000 & vanilla \& WGAN-GP \\ [0.5ex] 
\hline \hline
\multirow{3}{*}{\centering dbGaP} & 805 & 6500 & 5508 & 1000 & vanilla \&  WGAN-GP \\ \cline{2-6}
& 5000 & 6500 & 5508 & 1000 & vanilla \&  WGAN-GP\\ \cline{2-6}
& 10000 & 6500 & 5508 & 1000 & vanilla \& WGAN-GP \\ \cline{2-6} \hline
\end{tabular}
\label{tab:data-config}
\end{table}

\paragraph{$\detector$ Prediction Score.} \label{par:score}
In order to verify our Detectors succeed at the task for which they were trained, classifying samples from $\mixgandist$, in  Table ~\ref{tab:data-distin2} we report the test accuracy for our Detectors.
\begin{table}[H]
\centering
\setlength{\extrarowheight}{2.5pt} % Adjust cell height
\caption{Detector Test Accuracy}
\begin{tabular}{ ||c|c|c|c|c|c|| } 
\hline
Data & GAN Variant & SNPs Dim. & Test  Size &  Train Epochs & Mean Accuracy \\ 
\hline \hline
\multirow{3}{*}{\centering 1000  Genome} & vanilla  & 805 & 1000 & 159 & 0.974 \\ \cline{2-6}
& vanilla  & 5000 & 3000 & 240 & 0.980 \\ \cline{2-6}
& vanilla  & 1000 & 3000 & 360 & 0.999 \\ [0.5ex] 
\hline \hline
\multirow{3}{*}{\centering dbGaP} &  WGAN-GP & 805 & 1000 & 300 & 0.995 \\ \cline{2-6}
&  WGAN-GP & 5000 & 1000 & 600 & 0.998 \\ \cline{2-6}
&  WGAN-GP & 10000 & 6500 & 1500 & 0.930 \\ \cline{2-6} \hline
\end{tabular}
\label{tab:data-distin2}
\end{table}

\paragraph{Detector Architecture for Genomic Data Setting }\label{appendix:architecure}
Figure ~\ref{fig:side_by_side} shows the $\detector$ model architectures implemented in Keras \cite{keras} and Tensorflow \cite{tensorflow2015-whitepaper} for the MIAs on genomic GANs. The exact architecture differs depending on the dimension of the input data.

\begin{figure}[!hbtbp]
  \centering
  \subfloat[Architecture for 805 SNPs]{\includegraphics[width=0.3\textwidth]{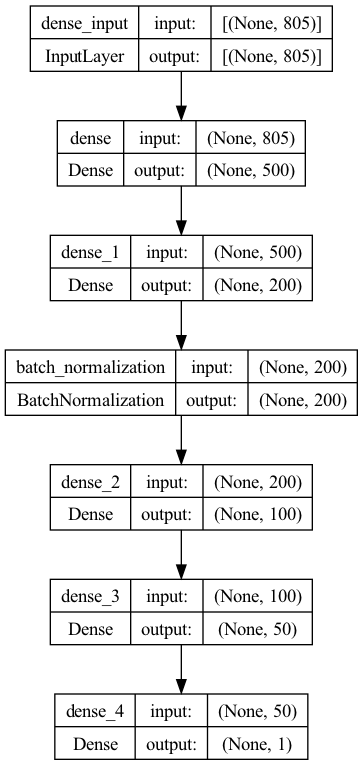}\label{fig:image1}}
  \hfill
  \subfloat[Architecture for 5000 SNPs]{\includegraphics[width=0.3\textwidth]{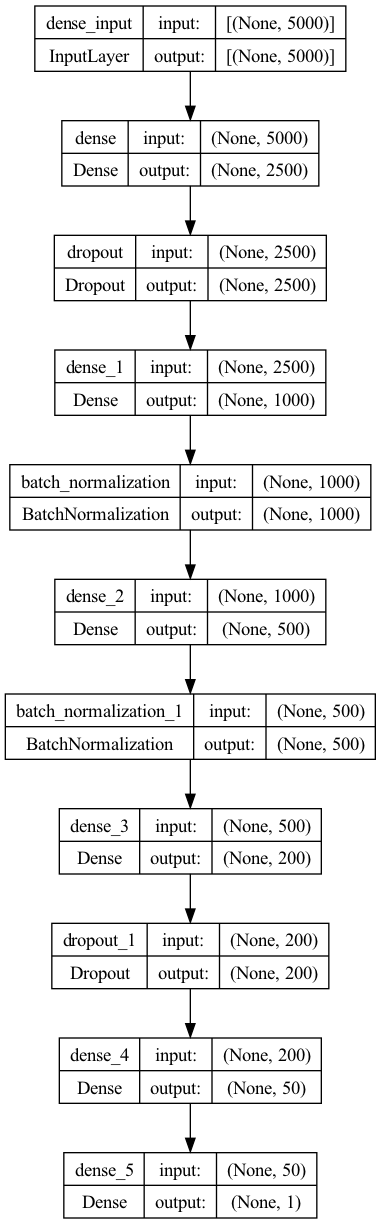}\label{fig:image2}}
  \hfill
  \subfloat[Architecture for 10000 SNPs]{\includegraphics[width=0.3\textwidth]{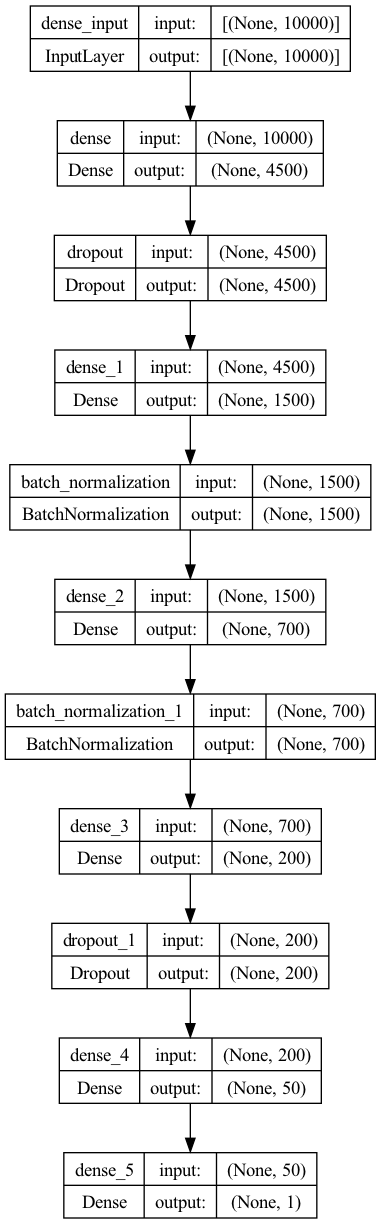}\label{fig:image3}}
  \caption{$\detector$ Architecture for MIA against  Vanilla GAN  trained on  1000  genome database for 805, 5000, and 10000 SNPS configurations. The $\detector$ architecture varies depending on the SNPs configurations.}
  \label{fig:side_by_side}
\end{figure}

\begin{figure}[!htbp]
  \centering
  \subfloat[Architecture for 805 SNPs]{\includegraphics[width=0.3\textwidth]{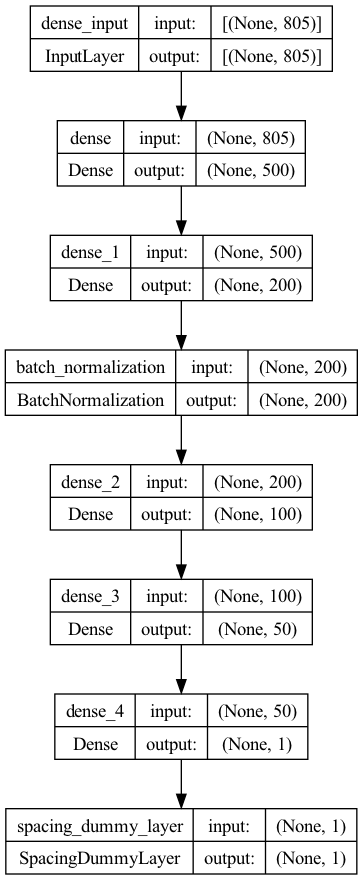}\label{fig:imagedb1}}
  \hfill
  \subfloat[Architecture for 5000 SNPs]{\includegraphics[width=0.3\textwidth]{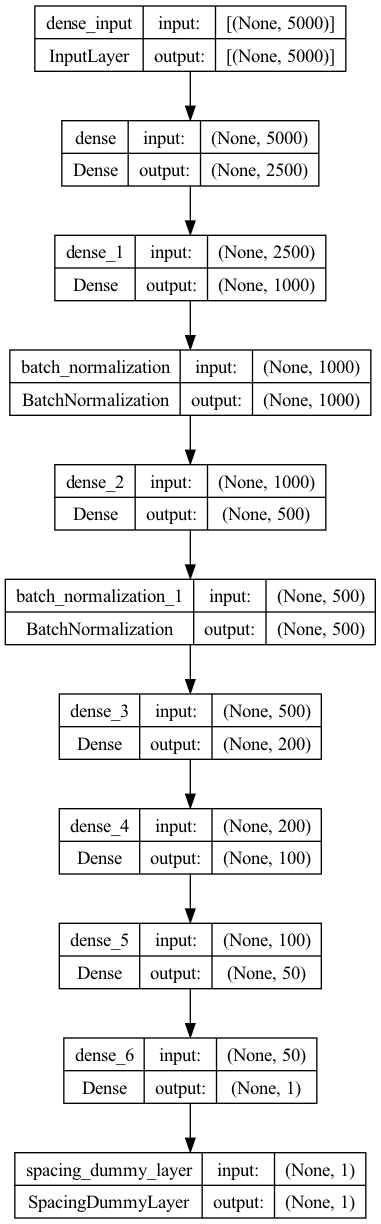}\label{fig:imagedb2}}
  \hfill
  \subfloat[Architecture for 10000 SNPs]{\includegraphics[width=0.3\textwidth]{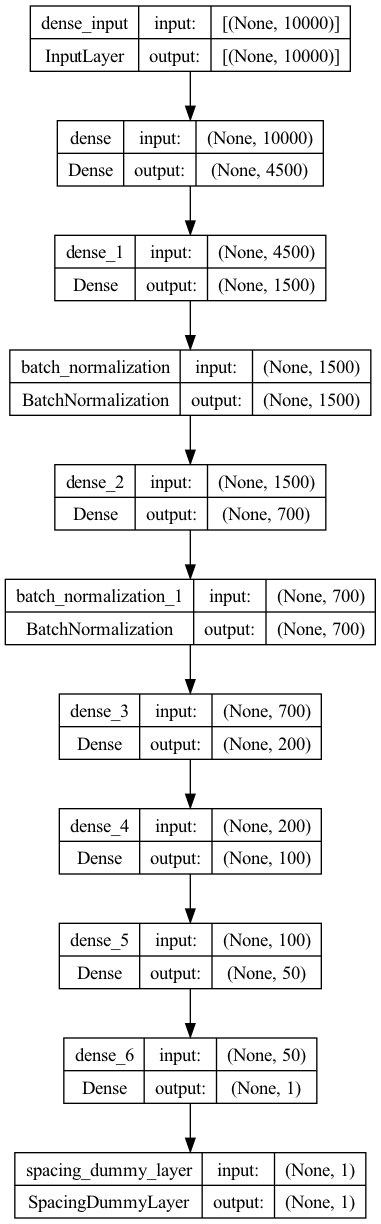}\label{fig:imagedb3}}
  \caption{$\detector$ Architecture for MIA against  Vanilla GAN  trained on  1000  genome database for 805,5000 and 10000 SNP dimension. The $\detector$ architecture varies depending on the SNPs configurations.}
  \label{fig:dbgapArch}
\end{figure}

\subsection{Augmented  Detector(ADIS)} \label{sec:adis}
ADIS builds on the $\detector$ by augmenting the feature space with variables derived from distance metrics and the DOMIAS attack. The pipeline is the same as that of training the Detector, with the following initial pre-processing step that creates the input to the network:
\begin{compactenum} 
    \item Sample $N(\approx 200)$ data points from $\datadistro, \gandist$ respectively. The sub-sampled data are for the computation of the reconstruction losses  (see Section~\ref{onedist}) and are not included again in the downstream training set.
    \item Apply dimensionality reduction separately to synthetic and reference samples, using PCA (details below).
    \item Given a test point $x \sim \mixgandist$, augment the feature space with test statistics computed using the projected and non-projected sample. We compute distance-based test statistics (one-way and two distances computed using the held-out sample),  and the DOMIAS likelihood ratio statistic $\frac{P_G(x)}{P_R(x)}$ ~\cite{refdata} where $P_G$ and $P_{X_R}$  are density models fitted on the synthetic and reference data respectively.
\end{compactenum}

\paragraph{ADIS Training Details.} \label{adis} The architecture for ADIS is the same as the $\detector$, but there are additional preprocessing steps and  fewer epochs ($\approx$ 10-30 epochs). A fixed sample size of $300$ from the reference and synthetic sample was set aside for the computation of reconstruction losses. This was followed by  fitting principal component analysis (PCA)  ~\cite{pca,naturepca} on the training data (consisting of reference and synthetic data) and selecting the first $100\text{-}300$ principal components - the actual component selected depends on the dimension of the original feature space, but typically $100$ components were selected for 805 SNPs and $300$ for $5000$ SNPs. Subsequently, we computed \textit{one-way and two-way} Hamming reconstruction losses on the non-transformed training data. Furthermore, as recommended in ~\cite{Mi-montecarlo}, we computed the \textit{two-way }  $\ell_2$ reconstruction loss on the PCA-transformed training data. Note that the $300$ samples we set aside had to be PCA-transformed before being used for the computation of \textit{two-way} $\ell_2$ reconstruction loss. For the computation of the test statistics of DOMIAS ~\cite{refdata}, we first reduced the dimensionality of the synthetic samples and reference data separately with the fitted PCA. Then we fit a Gaussian mixture model ~\cite{gausmix}, $P_G$, on the PCA-transformed synthetic samples and another  Gaussian mixture model, $P_{X_R}$, on the PCA-transformed reference data. Finally, we computed the test statistic $\frac{P_G(x)}{P_{X_R}(x)}$ for each PCA-transformed training point $x$.

% \paragraph{Large Reference  Data for ADIS}
% The ADIS attack is  a natural extension of the distinguisher attack to include attack variables from other attack models. We compare the performance of the ADIS to the distinguisher attack under different attack settings for genomic data, and the results show that ADIS performs better as shown in ~\ref{fig:reverse2,reversed,fig:attack2,fig:attack1}. 
% It should be noted that the $\detector$ and ADIS attacks would benefit from having a reference data sample of large size. In cases where the reference data size is  small, an effective strategy would be to train a secondary GAN on the reference data. Subsequently, one can proceed to subsample synthetic reference data (from the trained secondary GAN) that are closer or within an $\epsilon$-ball of the reference data as measured by a distance metric - i.e., using distance-based subsampling. 

\paragraph{Training  DOMIAS}\label{domias_explained}
Training DOMIAS involves 2 steps -- dimensionality reduction and density estimation. In the genomic data setting, we used PCA for dimension reduction following the same steps as described for ADIS.  For density estimation, we using the non-volume preserving transformation (NVP)\cite{densityEstimation}. We fit the the estimator to the synthetic and reference samples separately. For image data, we projected the image onto a low dimensional sub-space using the encoder of a VAE trained on the reference and synthetic data. Subsequently, we fit the density estimator using NVP on the transformed data.

\section{Convergence Plots And Memorization}
\label{sub:overfit}
We present PCA convergence plots for GANs trained on genomic data and empirical analysis of memorization in these GANs.
\subsection{Genomic GAN Convergence Plots}
\label{sec:gen_convergence}
For image data, visual inspection is a quick way to examine if the synthetic samples have converged to the underlying training samples. Clearly, with high-dimensional tabular data such quick visual convergence examination does not work. Recent work \cite{platzer2013visualization}  proposed using both the principal component analysis (PCA) and $t$-distributed stochastic neighbor embedding (T-SNE) plots for  the analysis of genomic data population structure. Consequently, we examine the convergence of the synthetic samples using both PCA and T-SNE. In particular, the synthetic samples converge if the population structure as captured by the PCA and T-SNE   is similar to that of the corresponding real genomic samples. 
% We examine  the PCA and T-SNE plots for both the synthetic samples and the real training samples. 
Figures \ref{fig:p1} - \ref{fig:p6}  depict  the 6 principal components of both the training data and synthetic samples for each configuration  in Table \ref{tab:data-config}. Based on the plots it appears the synthetic samples are indeed converging towards the underlying data distribution.
\begin{figure}[htbp!]
\centering
\renewcommand\arraystretch{0.5}
\begin{tabular}{@{} c|c|c @{}}
  \subfloat[  805 SNPS, 1k Genome]{\includegraphics[width=.32\textwidth]{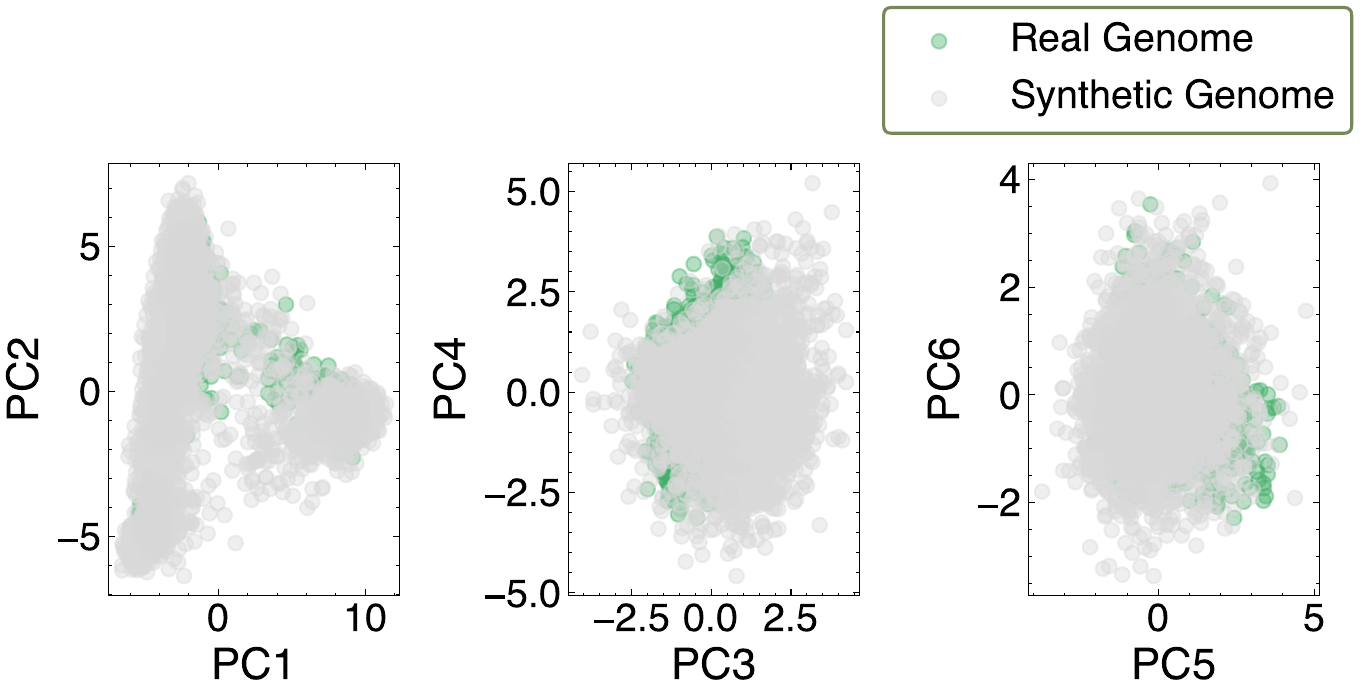}\label{fig:p1}}%
&
 \subfloat[  $5$k SNPS, 1k Genome]{\includegraphics[width=.32\textwidth]{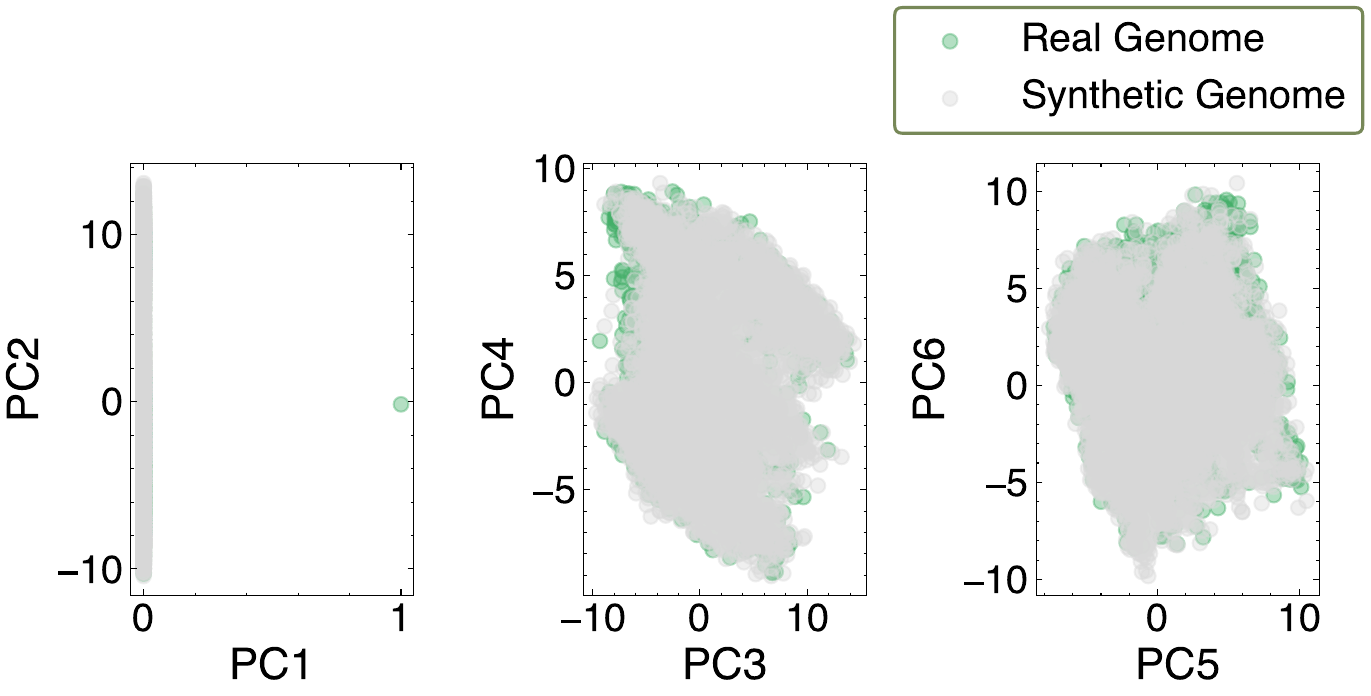}\label{fig:p2}}%
 &
 \subfloat[  10 SNPS, 1k Genome]{\includegraphics[width=.32\textwidth]{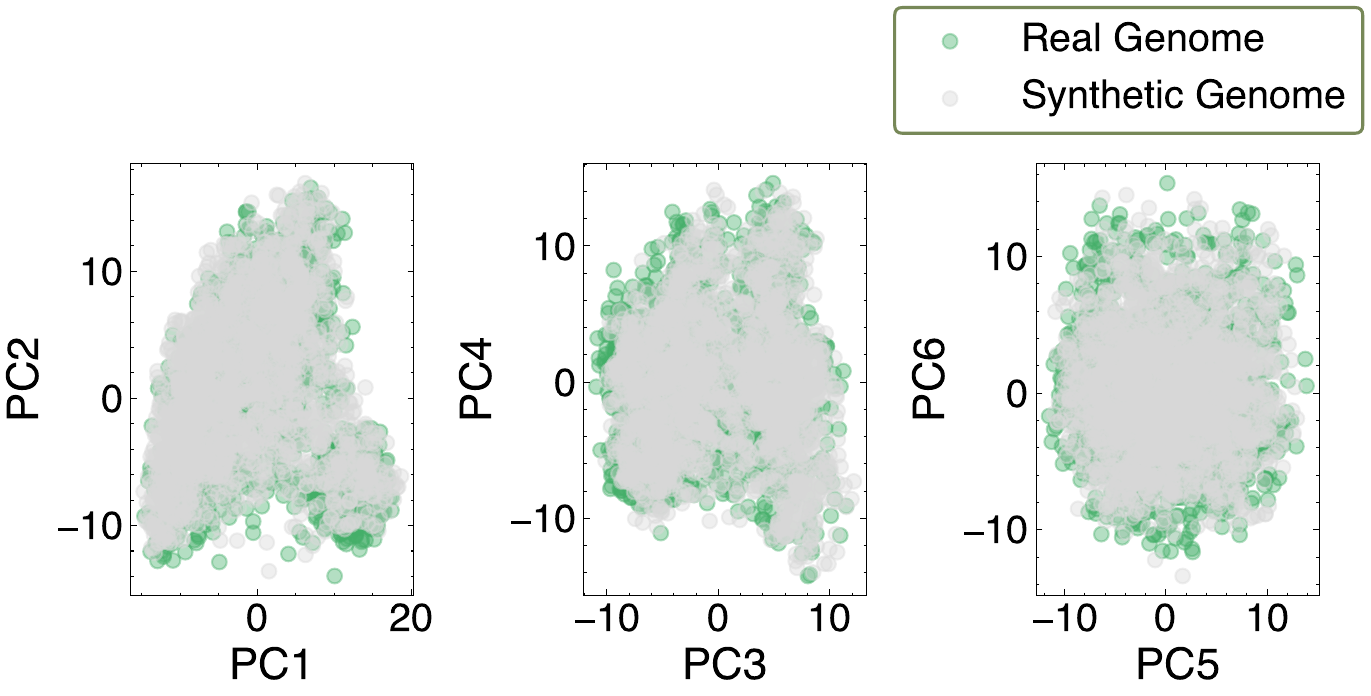}\label{fig:p3}}%
\\
& \\
\hline
& \\
 \subfloat[  805 SNPS, dbGaP]{\includegraphics[width=.32\textwidth]{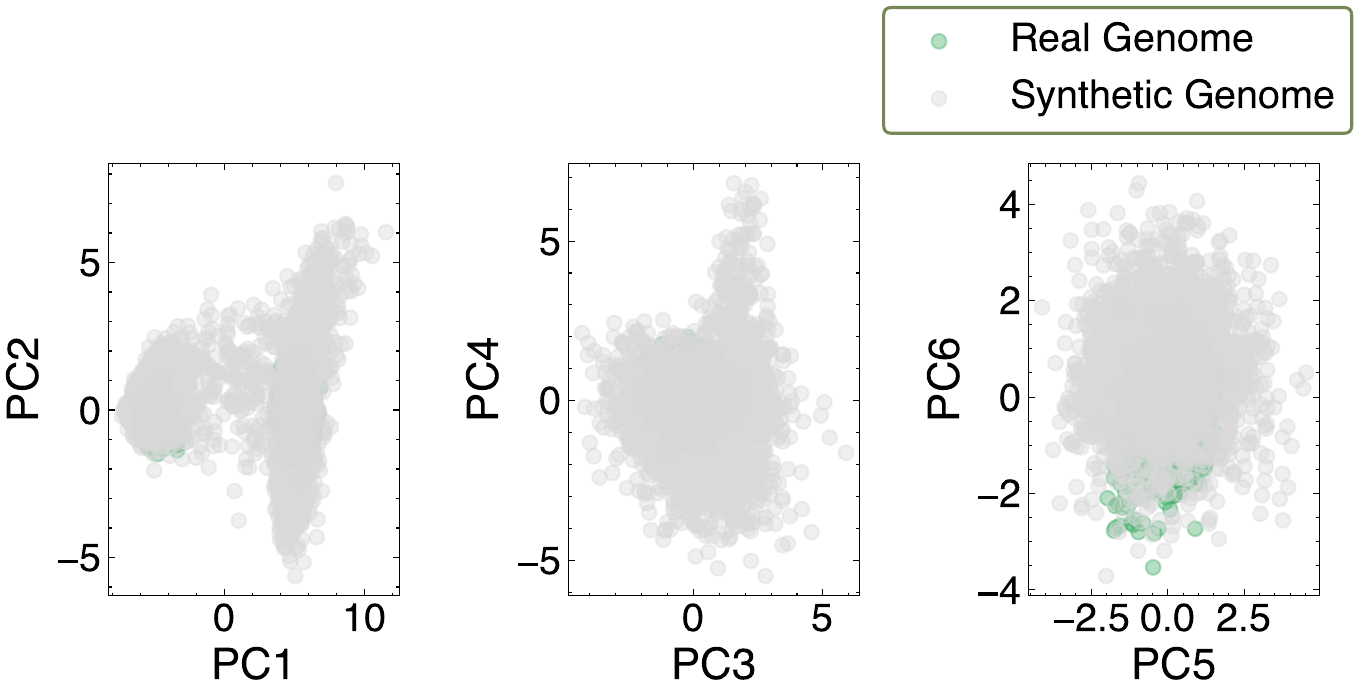}\label{fig:p4}}%
&
 \subfloat[  $5$k SNPS, dbGaP]{\includegraphics[width=.32\textwidth]{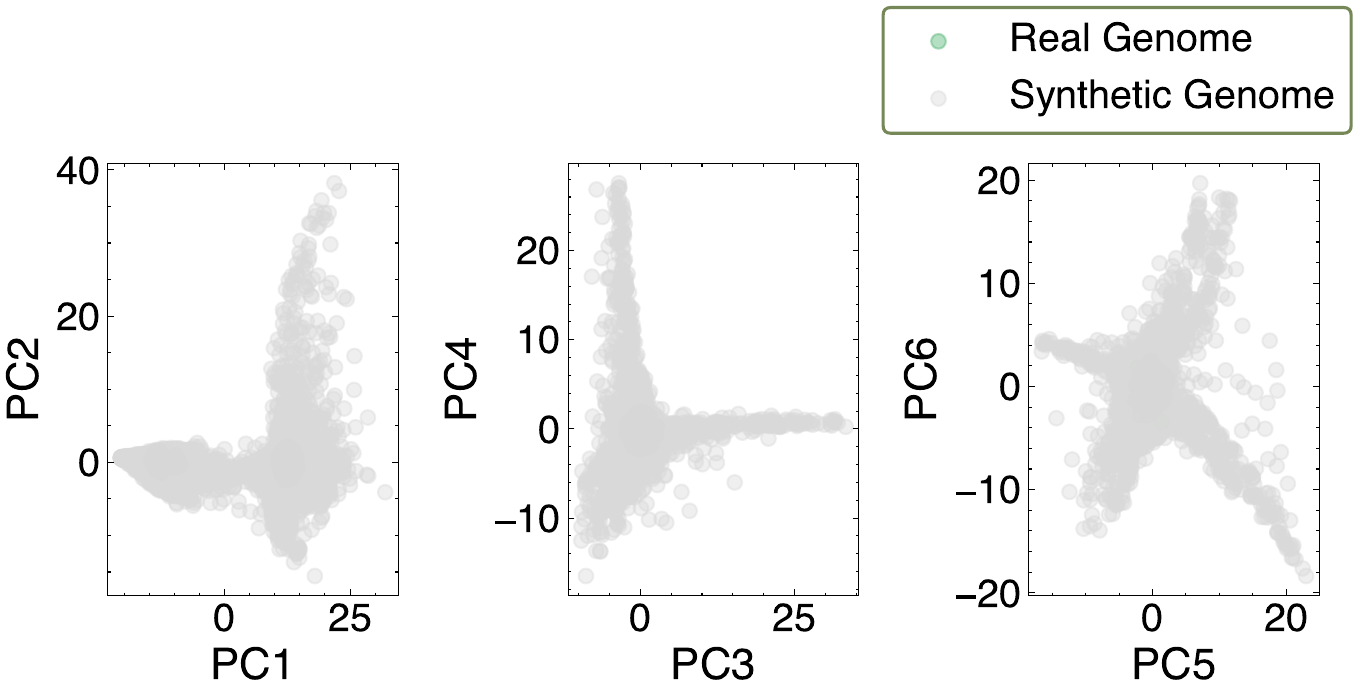}\label{fig:p5}}%
 &
 \subfloat[  $10$k SNPS, dbGaP]{\includegraphics[width=.32\textwidth]{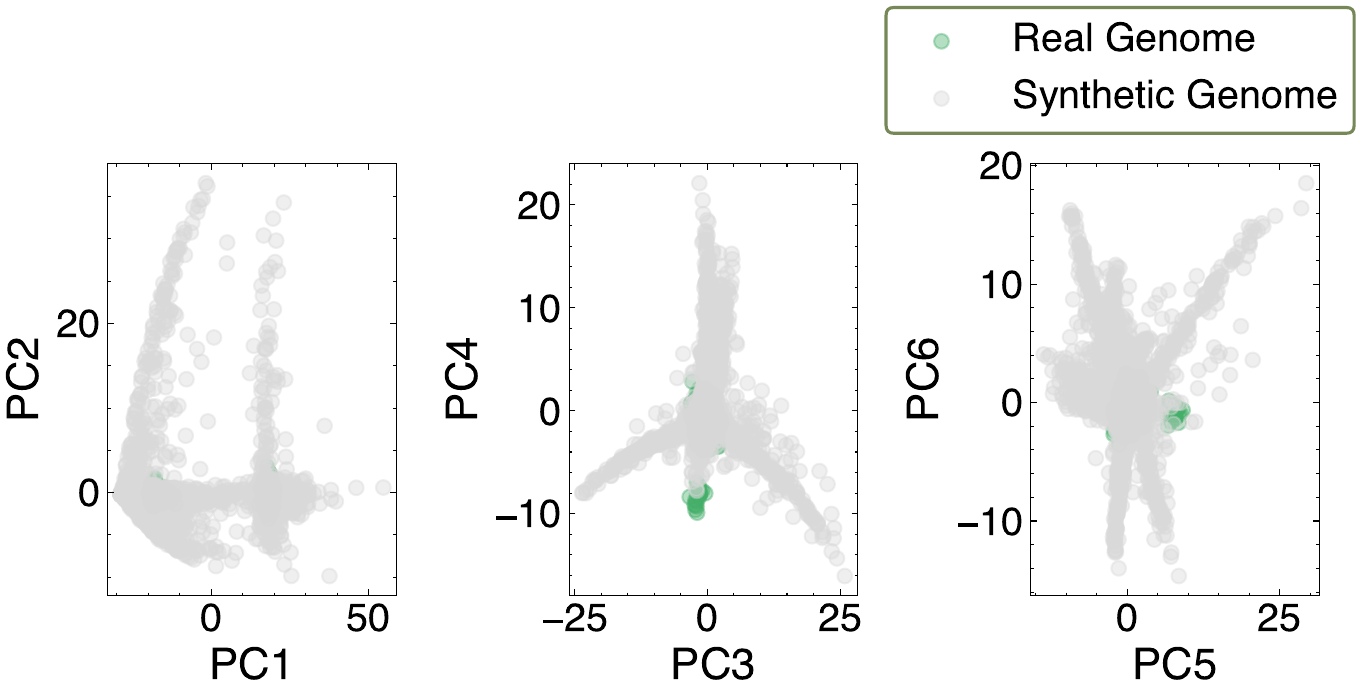}\label{fig:p6}}%
\end{tabular}
\caption{ PCA Plots for all data configurations. The plot could be read row-wise or column-wise. The first row depicts PCA plots for synthetic samples from Vanilla GAN trained on 1000 Genome data for $805$, $5$k, and $10$k SNPS respectively. The second row depicts PCA plots for synthetic samples from WGAN  trained on dbGaP data for $805$, $5$k, and $10$k SNPs respectively. Each column corresponds to 805,$5$k and $10$k SNPS configurations. }
\end{figure}
\begin{figure}[htbp!]
\centering
\begin{tabular}{@{} c|c @{}}
  \subfloat[$5$k SNPS, 1k Genome]{\includegraphics[width=.25\textwidth]{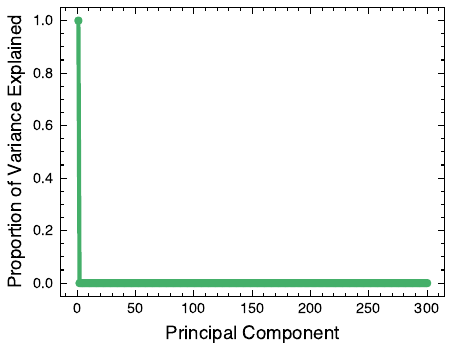}\label{fig:sc1}}%
\subfloat[$5$k SNPS, 1k Genome]{\includegraphics[width=.25\textwidth]{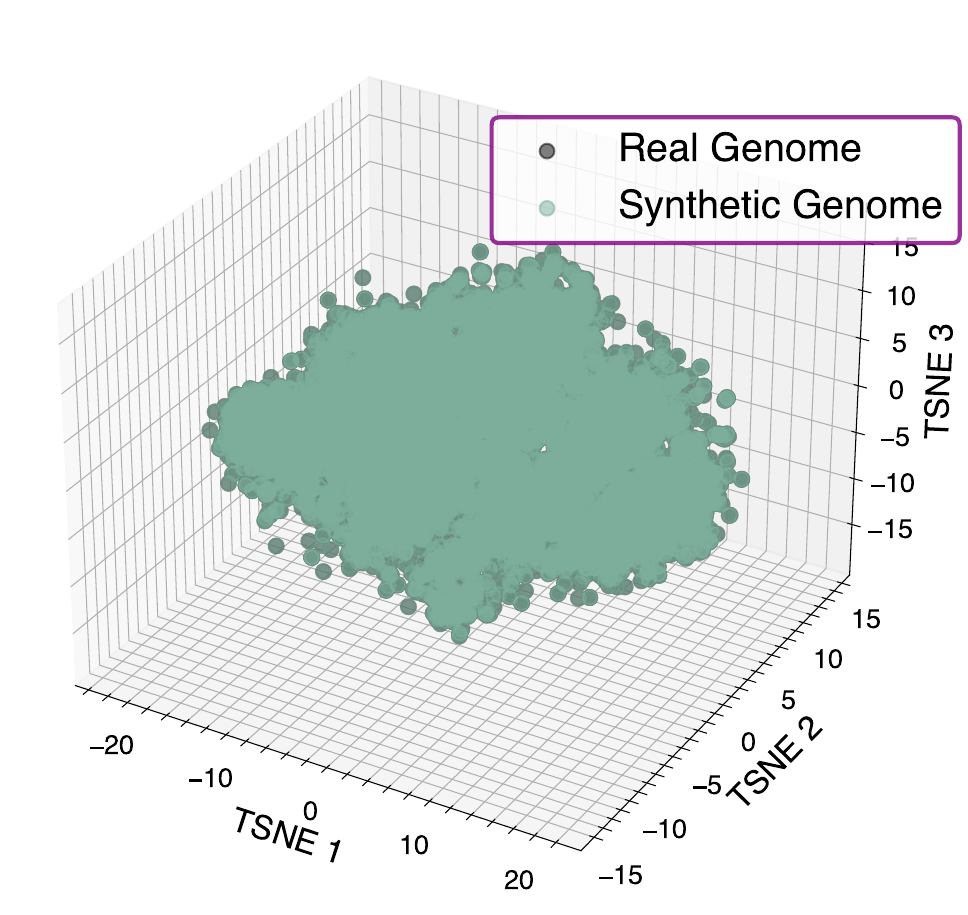}\label{fig:sc2}}%
&
 \subfloat[$10$k SNPS, 1k Genome]{\includegraphics[width=.25\textwidth]{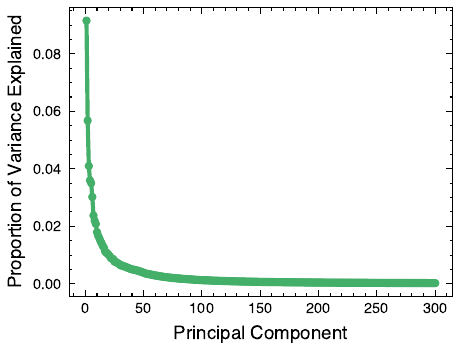}\label{fig:sc3}}%
\subfloat[$10$k SNPS, 1k Genome]{\includegraphics[width=.25\textwidth]{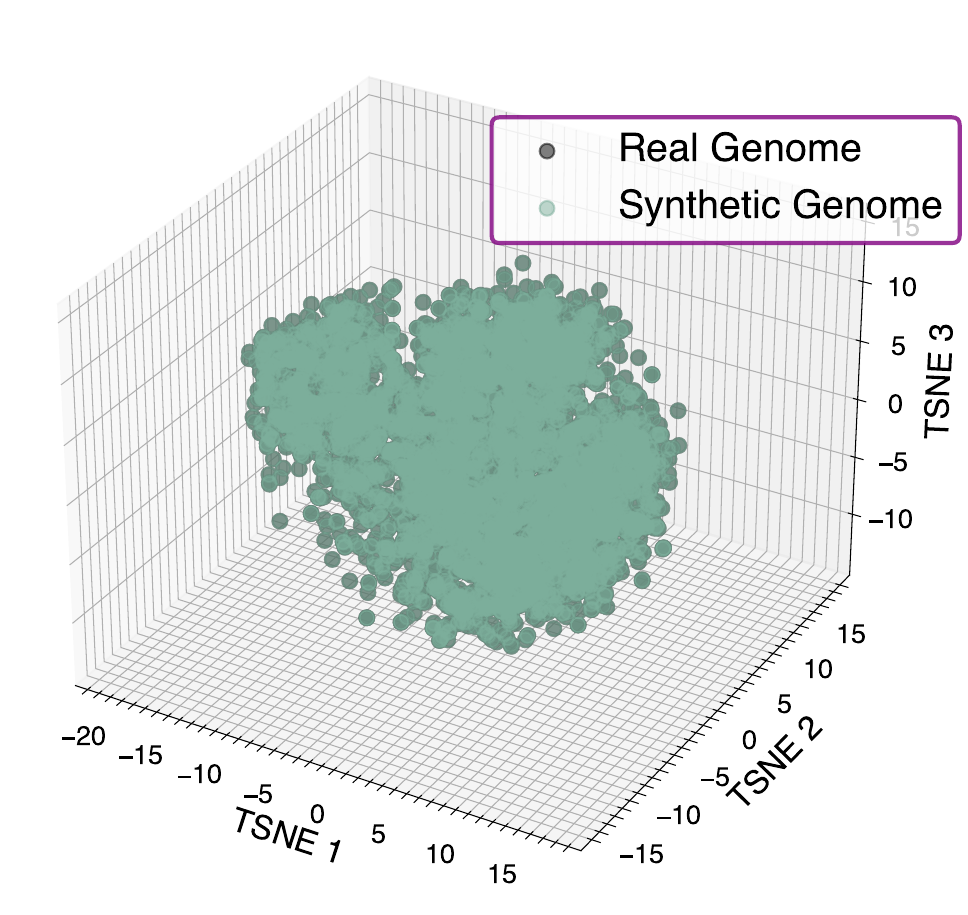}\label{fig:sc4}}%
\\
& \\
\hline
& \\
 \subfloat[$5$k SNPS, dbGaP]{\includegraphics[width=.25\textwidth]{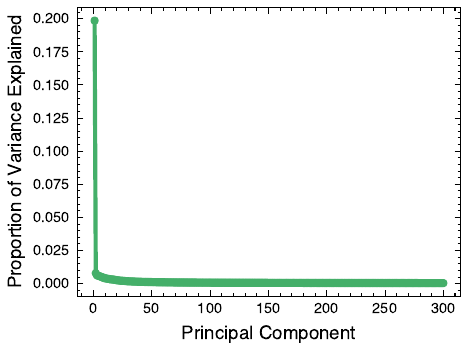}\label{fig:sc5}}%
 \subfloat[ $10$k SNPS,]{\includegraphics[width=.25\textwidth]{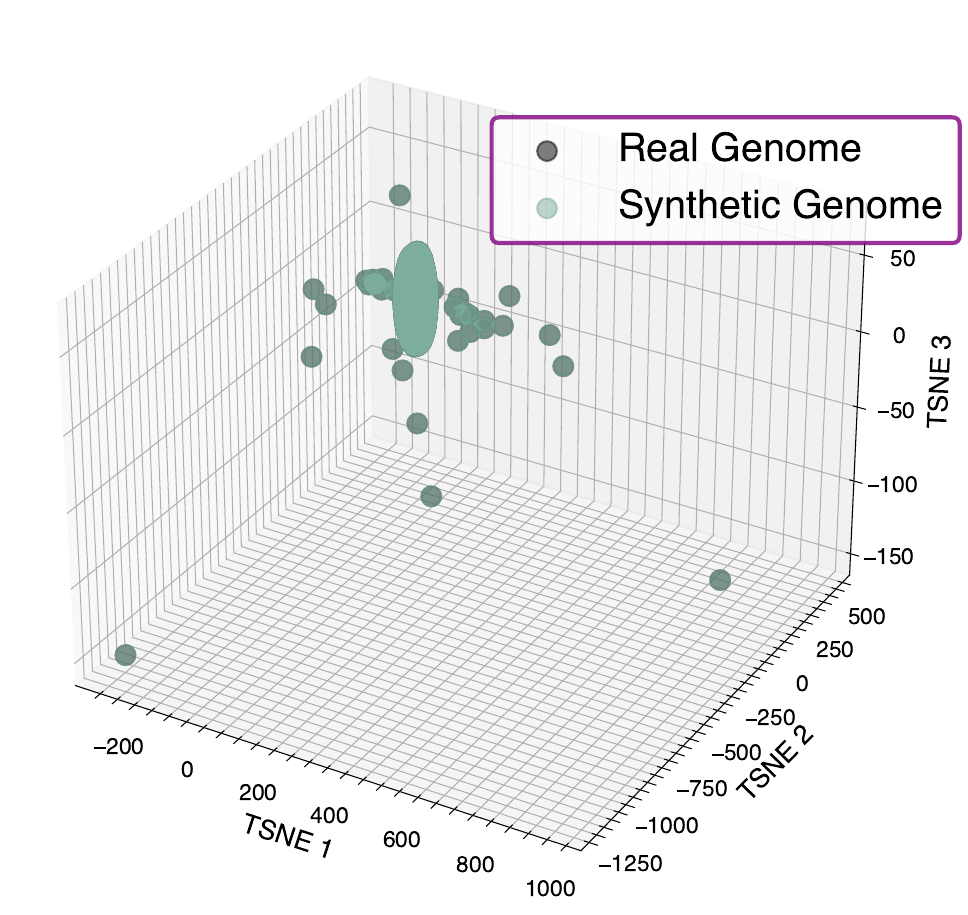}\label{fig:sc6}}%
&
 \subfloat[$10$k SNPS, dbGaP ]{\includegraphics[width=.25\textwidth]{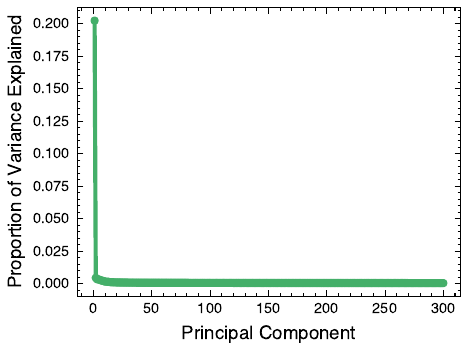}\label{fig:sc7}}%
 \subfloat[
$10$k SNPS, dbGaP]{\includegraphics[width=.25\textwidth]{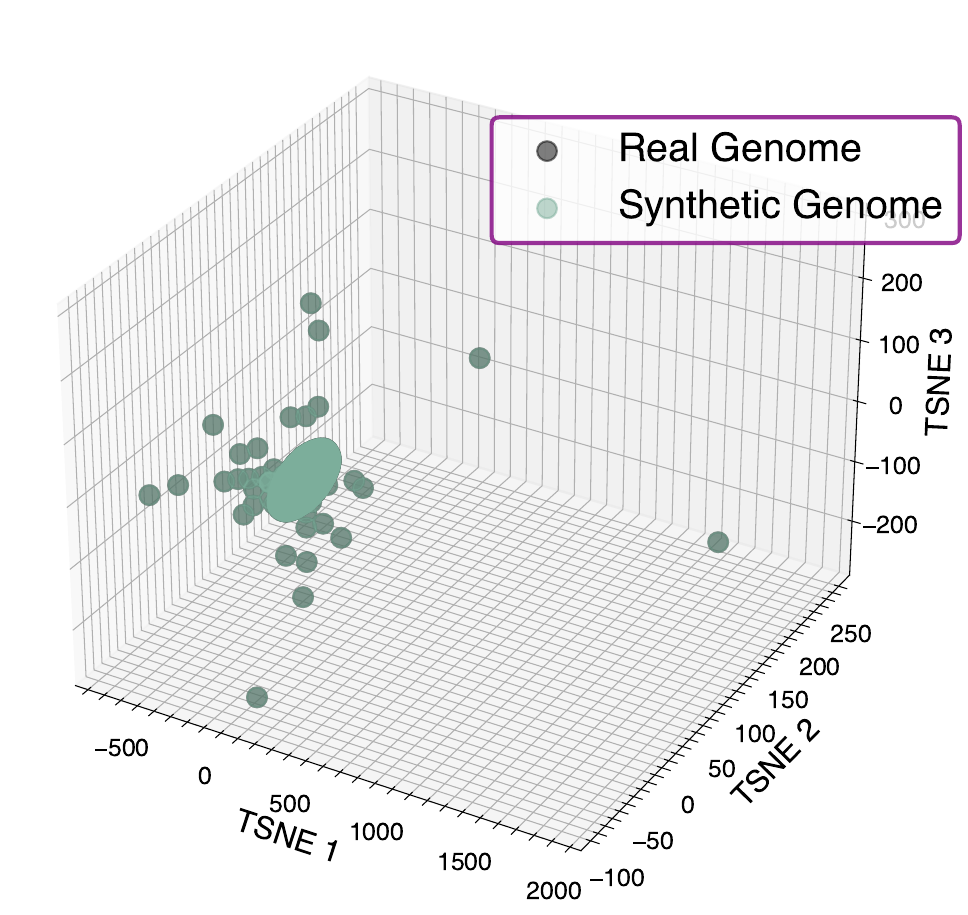}\label{fig:sc8}}%
\end{tabular}
\caption{The first row shows scree and TSNE  plots for synthetic samples from Vanilla GAN trained on 1000 Genome data for  $5$k, $10$k SNPS respectively. The second row shows scree and TSNE  plots for synthetic samples from WGAN trained on dbGaP  data for  $5$k, and $10$k SNPS respectively. The distance metric for the reconstruction loss is the Hamming distance. }
\end{figure}
% To , given $\datadistro$, we compute $s_i \in S'$,  where $s_i$  is the nearest synthetic sample to $d_i \in \datadistro $ as measured by a distance metric. The test returns the $p$-value under the null hypothesis that $\datadistro$ and $S'$ are the same distribution.

% we performed a kernel 2 sample test \cite{JMLR:v13:gretton12a} on the training set and their nearest neighbor synthetic samples to under the
% null hypothesis that the 2 samples are the same.
% \subsection{Memorization Check}

% \textbf{Memorization}: We also carried out test for  \textit{whole sequence memorization} and \textit{subsequence memorization} - our experimental data set is genomic data and it is a  subset of $\{  0,1\}^d$, where $d = \{ 805,5000, 10000  \}$, so that it is natural  to check for whole sequence and subsequence memorizations. 

\subsection{Over-fitting and Memorization} \label{wholeseq}
This section performs some preliminary analyses on the genomic data GANs to ensure they aren't memorizing the training set. Memorization is one indicator of overfitting, which would artificially inflate the attack success metrics of our MIAs. 
\paragraph{Whole sequence memorization.} We iteratively sample batches of  $3000$-$4000$ synthetic samples, $S_B$, from  the GAN until the sample  size is of order $\geq 10^6$. For each sample point in the $i$-th batch, $s_{b} \in  S_{B_i}$ , we compute its minimum Hamming distance to  $\datadistro$, the training data. Observe that a Hamming distance of zero would indicate whole sequence memorization - since it implies that an exact copy of a training data sample  is being synthesized by the GAN. Figures \ref{fig:memorize}, \ref{fig:memorize2} show the GAN models do not memorize the training data, since no reconstruction losses of zero are reported.

\sn{Whole sequence memorization results can go in the appendix, not that interesting}
\luk{this is done.}
\begin{figure}[htpb!]%
 \centering
 \subfloat[805 SNPS,1k Genome]{\includegraphics[width=.32\textwidth]{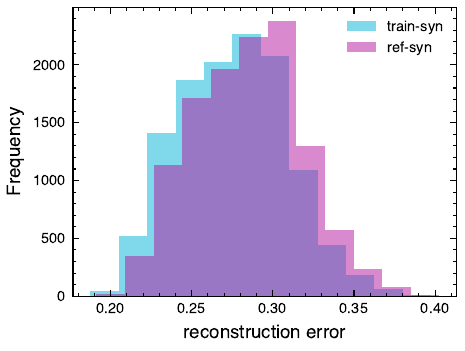}\label{fig:recon1}}%
 \subfloat[$5$k SNPS,1k Genome]{\includegraphics[width=.32\textwidth]{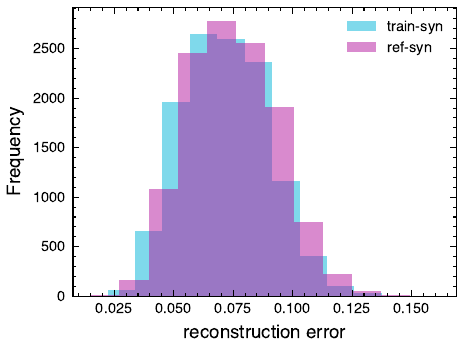}\label{fig:recon2}}
 \subfloat[$10$k SNPS,1k Genome]{\includegraphics[width=.32\textwidth]{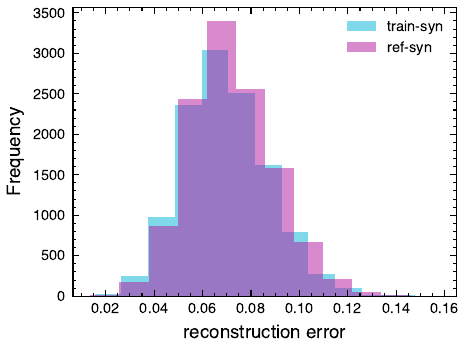}\label{fig:recon3}}%
 \caption{ The plots show the frequency distribution of reconstruction error for  whole sequence memorization test  for 3 batches of synthetic samples from vanilla GAN trained. The label \textbf{train-syn} (in seafoam green) indicates that we are measuring reconstruction loss for  each batch of the synthetic samples with respect to the training data.  For reference purposes, we also plot, \textbf{ref-syn}, which is the reconstruction loss of the synthetic samples  given the reference samples (in purple). The distance metric for the reconstruction loss is the Hamming distance.}%
 \label{fig:memorize}
\end{figure}
\begin{figure}[htpb!]%
 \centering
 \subfloat[805 SNPs, dbGaP]{\includegraphics[width=.32\textwidth]{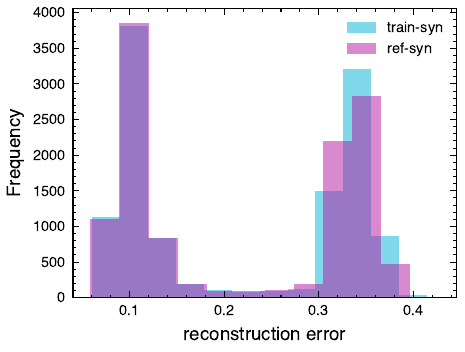}\label{fig:recon4}}%
 \subfloat[$5$k SNPs, dbGap]{\includegraphics[width=.32\textwidth]{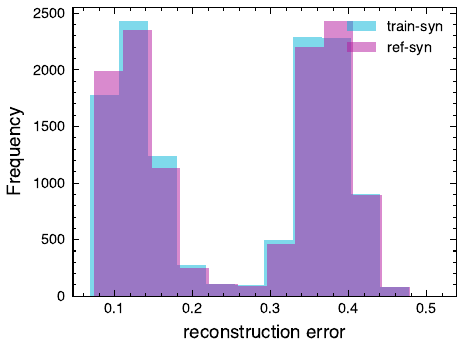 }\label{fig:recon5}}
 \subfloat[$10$k SNPs, dbGaP]{\includegraphics[width=.32\textwidth]{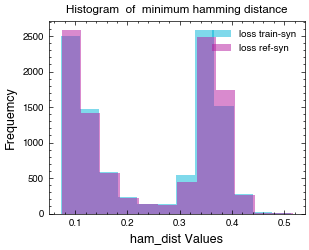}}
\label{memorize1} %
 \caption{The plots show the frequency distribution of reconstruction error for 3 batches of synthetic samples from WGAN-GP. The label \textbf{train-syn} (in seafoam green) indicates that we are measuring reconstruction loss for  each batch of the synthetic samples with respect to the training data .  For reference purposes, we  plot \textbf{ref-syn}, which is the reconstruction loss of the synthetic samples  given the reference samples (in purple). The distance metric for the reconstruction loss is the Hamming distance.  }%
 \label{fig:memorize2}%
\end{figure}

% \section{Feature  Importance }

% Feature importance for distinguisher model trained on reference data and synthetic samples from target GAN.

% \begin{figure}[!htbp]%
%  \centering
%  \subfloat[Feature Importance]{\includegraphics[width=.42\textwidth,height=.30\textheight]{media/meanfreq.pdf}}\label{fig:mf}%
%   \qquad
%  \subfloat[mean frequency of SNPs]{\includegraphics[width=.42\textwidth,height=.30\textheight]{media/shapImportance.pdf}}\label{fig:mfb}
%  \caption{ Fearture Importance as measured by SHAP value for each SNPs  and importance as measured by average SNPs value  }%
%  \label{fig:importance}%
% \end{figure}

\sn{Do we need the below? Appendix? What does it add to the paper?}
\section{Results Table for Genomic GANs (TPR At Low FPR)}
Table \ref{tab:summary} depicts the TPR at FPR of $0.01$ , $0.1$, $0.005$ and $0.001$ respectively (see   Figures \ref{fig:roc_gen} and \ref{fig:rev1}) for genomic data setting.

\begin{table}[!htbp]
\centering
\setlength{\extrarowheight}{3pt} % Adjust cell height
\caption{Table of Attack Results (Genomic Data)}
\begin{tabular}{ ||c|c|c|p{1.5cm}|p{1.5cm}|p{1.5cm}|p{1.5cm}|| } 
\hline
GAN Variant & Dataset & Attack Method & TPR @0.01 FPR & TPR @ 0.1 FPR & TPR @0.005 FPR  & TPR @0.001 FPR  \\ 
\hline \hline
vanilla 805 & 1000 genome & one-way distance  & 1.02\%  & 6.95\% & 0.51\% & 0.13\% \\ 
\hline
vanilla 805 & 1000 genome & two-way distance & 7.79\%  & 28.86\% & 5.14\% & 1.89\%  \\ 
\hline
vanilla 805 & 1000 Genome & weighted distance & 4.75\% & 21.76
\% & 2.85\% & 1.76\% \\ 
\hline
vanilla 805 & 1000 Genome & Robust Homer & 2.46\% & 17.89\% & 1.26\% & 0.20\% \\ \hline 
vanilla 805 & 1000 Genome & $\detector$ & 1.89\% & 13.72\% & 0.97\% & 0.19\% \\ \hline
vanilla 805 & 1000 Genome & ADIS & 5.31\% & 26.89\% & 3.14\% & 1.16\% \\ 
\hline 
vanilla $5$k & 1000 genome & one-way distance  & 0.88\% & 7.6\% & 0.37\% & 0.10\% \\ 
\hline
vanilla $5$k & 1000 genome & two-way distance & 1.36\% & 13.89\% & 0.68\% & 0.25\% \\ 
\hline
vanilla $5$k & 1000 Genome & weighted distance & 1.00\% & 12.6\% & 0
.56\% & 0.15\% \\ 
\hline
vanilla $5$k & 1000 Genome & Robust Homer & 0.88\% & 7.62\% & 0.58\% & 0.25\% \\ \hline
vanilla $5$k & 1000 Genome & $\detector$ & 1.52\% & 12.70\% & 0.82\% & 0.33\% \\  \hline
vanilla $5$k & 1000 Genome &ADIS & 2.00\% & 13.40\% & 0.99\% & 0.26\% \\ 
\hline 
vanilla $10$k & 1000 genome & one-way distance  & 1.10\% & 9.97\% & 0.62\% & 0.23\% \\ 
\hline
vanilla $10$k & 1000 genome & two-way distance & 1.07\% & 10.52\% & 0.42\% & 0.17\% \\ 
\hline
vanilla $10$k & 1000 Genome & weighted distance & 0.90\% & 10.4\% & 0.49\% & 0.15\% \\ 
\hline
vanilla $10$k & 1000 Genome & Robust Homer & 1.11\% & 10.56\% & 0.58\% & 0.36\% \\ \hline
vanilla $10$k & 1000 Genome & $\detector$ & 1.38\% & 11.56\% & 0.68\% & 0.17\% \\ \hline
vanilla $10$k & 1000 Genome & ADIS & 1.72\% & 12.56\% & 1.03\% & 0.25\% \\ 
\hline \hline

wgan-gp 805 & dbGaP & one-way distance  & 0.91\% & 7.69\% & 0.54\% & 0.23\% \\ 
\hline
wgan-gp 805 & dbGaP & two-way distance & 1.52\% & 11.91\% & 0.63\% & 0.23\% \\ 
\hline
wgan-gp 805 & dbGaP & weighted distance & 1.41\% & 11.75\% & 0.91\% & 0.38\% \\ 
\hline
wgan-gp 805 & dbGaP & Robust Homer & 1.37\% & 11.22\% & 0.67\% & 0.25\% \\ 
\hline
wgan-gp 805 & dbGaP & $\detector$ & 2.59\% & 16.01\% & 1.14\% & 0.22\% \\ 
\hline 
wgan-gp 805 & dbGaP & ADIS & 2.40\% & 15.15\% & 1.04\% & 0.35\% \\ 
\hline 
wgan-gp $5$k & dbGaP & one-way distance & 0.93\% & 9.17\% & 0.23\% & 0.16\% \\ \hline
wgan-gp $5$k & dbGaP & two-way distance & 1.62\% & 12.81\% & 0.77\% & 0.29\% \\ 
\hline
wgan-gp $5$k & dbGaP & weighted distance & 1.90\% & 12.90\% & 1.90\% & 0.35\% \\ 
\hline
wgan-gp $5$k & dbGaP & Robust Homer & 2.13\% & 14.03\% & 1.48\% & 0.62\% \\ 
\hline
wgan-gp $5$k & dbGaP & $\detector$ & 2.95\% & 20.10\% & 1.64\% & 0.37\% \\ 
\hline 
wgan-gp $5$k & dbGaP & ADIS & 6.40\% & 28.70\% & 3.98\% & 1.32\% \\ 
\hline 
wgan-gp $10$k & dbGaP & one-way distance  & 1.42\% & 9.38\% & 0.55\% & 0.09\% \\ 
\hline
wgan-gp $10$k & dbGaP & two-way distance & 2.07\% & 12.45\% & 1.20\% & 0.47\% \\ 
\hline
wgan-gp $10$k & dbGaP & weighted distance & 1.96\% & 11.80\% & 0.98\% & 0.45\% \\ 
\hline
wgan-gp $10$k & dbGaP & Robust Homer & 2.71\% & 13.95\% & 1.35\% & 0.65\% \\ 
\hline
wgan-gp $10$k & dbGaP & $\detector$ & 3.26\% & 24.23\% & 2.11\% & 0.67\% \\ \hline
wgan-gp $10$k & dbGaP & ADIS & 6.07\% & 24.49\% & 4.24\% & 1.6\% \\ 
\hline
\end{tabular}
\label{tab:summary}
\end{table}

\section{Results Table for Image GANs (TPR At Low FPR)}
Table \ref{tab:summary2} depicts the TPR at FPR of $0.01$ , $0.1$, $0.005$ and $0.001$ respectively (see   Figures \ref{fig:roc_gen} and \ref{fig:rev1}) for genomic data setting.

\begin{table}[!htbp]
\centering
\setlength{\extrarowheight}{3pt} % Adjust cell height
\caption{Table of Attack Results (Image Data)}
\begin{tabular}{ ||c|c|c|p{1.5cm}|p{1.5cm}|p{1.5cm}|p{1.5cm}|| } 
\hline
GAN Variant & Dataset & Attack Method & TPR @0.01 FPR & TPR @0.1 FPR & TPR @0.005 FPR & TPR @0.001 FPR \\ 
\hline \hline
BigGAN & CIFAR10 & two-way distance ($\ell_2$)  & 1.0\% & 10.8\% & 0.6\% & 0.6\%\\ 
\hline
BigGAN & CIFAR10  & one-way distance ($\ell_2$) & 0.6\% & 8.7\% & 0.2\% & 0.0\% \\ 
\hline
BigGAN & CIFAR10  & two-way distance (\texttt{LPIPS}) & 0.4\% & 11.0\% & 0.1\% & 0.1\% \\ 
\hline
BigGAN & CIFAR10  & one-way distance (\texttt{LPIPS}) & 1.2\% & 10.8\% & 0.7\% &  0.1\% \\ 
\hline
BigGAN &  CIFAR10 & $\detector$ (Conv. net) & 0.8\% & 7.7\% & 0.3\% & 0.0\% \\ 
\hline
BigGAN & CIFAR10 & $\detector$ (syn-syn) &  1.4\% & 11.8\% & 0.7\% & 0.2\% \\ 
\hline
BigGAN & CIFAR10 & $\detector$(Incept.-MLP) & 2.2\% & 9.3\% & 1.6\% & 0.2\% \\ 
\hline
DCGAN & CIFAR10 & two-way distance ($\ell_2$)  & 1.0\% & 10.8\% & 0.6\% & 0.6\% \\ 
\hline
DCGAN & CIFAR10  & one-way distance ($\ell_2$) & 1.0\% & 11.1\% & 0.6\% & 0.6\% \\ 
\hline
DCGAN & CIFAR10  & two-way distance (\texttt{LPIPS}) & 1.2\% & 13.0\% & 0.7\% & 0.3\% \\ 
\hline
DCGAN & CIFAR10  & one-way distance (\texttt{LPIPS}) & 1.7\% & 9.0\% & 0.6\% & 0.2\% \\ 
\hline
DCGAN &  CIFAR10 & $\detector$ (Conv. net) & 1.4\% & 9.7\% & 1.1\% & 0.3\% \\ 
\hline
DCGAN & CIFAR10 & $\detector$ (syn-syn) & 1.2\% & 9.6\% & 0.6\% & 0.1\% \\ 
\hline
DCGAN & CIFAR10 & $\detector$ (Incept.-MLP) & 0.8\% & 11.4\% & 0.4\% & 0.2\% \\ 
\hline
ProjGAN & CIFAR10 & two-way distance ($\ell_2$)  & 1.0\% & 10.7\% & 0.6\% & 0.6\% \\ 
\hline
ProjGAN & CIFAR10  & one-way distance ($\ell_2$) & 0.6\% & 8.4\% & 0.3\% & 0.1\% \\ 
\hline
ProjGAN & CIFAR10  & two-way distance (\texttt{LPIPS}) & 0.7\% & 10.7\% & 0.6\% & 0.1\% \\ 
\hline
ProjGAN & CIFAR10  & one-way distance (\texttt{LPIPS}) & 1.0\% & 10.9\% & 0.8\% & 0.6\% \\ 
\hline
ProjGAN &  CIFAR10 & $\detector$ (Conv. net) & 1.5\% & 9.9\% & 0.7\%  & 0.1\%  \\ 
\hline
ProjGAN & CIFAR10 & $\detector$ (syn-syn) & 2.0\% & 11.8\% & 1.4\% & 0.1\% \\ 
\hline
ProjGAN & CIFAR10 & $\detector$ (Incept.-MLP) & 0.9\% & 11.6\% & 0.2\% & 0.0\% \\ 
\hline
ContraGAN & CIFAR10 & two-way distance ($\ell_2$)  & 1.0\% & 10.8\% & 0.6\% & 0.6\% \\ 
\hline
ContraGAN & CIFAR10  & one-way distance ($\ell_2$) & 0.3\% & 8.7\% & 0.3\% & 0.0\% \\ 
\hline
ContraGAN & CIFAR10  & two-way distance (\texttt{LPIPS}) & 1.7\% & 12.3\% & 1.0\% & 0.3\% \\ 
\hline
ContraGAN & CIFAR10  & one-way distance (\texttt{LPIPS}) & 1.3\% & 10.4\% & 0.4\% & 0.3\% \\ 
\hline
ContraGAN &  CIFAR10 & $\detector$ (Conv. net) & 1.2\% & 9.6\% & 0.7\% & 0.1\% \\ 
\hline
ContraGAN & CIFAR10 & $\detector$ (syn-syn) & 0.9\% & 10.0\% & 0.5\% & 0.01\% \\ 
\hline
ContraGAN & CIFAR10 & $\detector$ (Incept.-MLP) & 0.9\% & 10.4\% & 0.3\% & 0.2\% \\ 
\hline
\end{tabular}
\label{tab:summary2}
\end{table}

\end{document}